\pdfoutput=1
\def\isarxiv{1} 
\ifdefined\isarxiv
\documentclass[11pt]{article}
\else
\documentclass{article}
\usepackage{hyperref}
\usepackage{neurips_2021}
\fi

\usepackage{amsmath}
\usepackage{amsthm}
\usepackage{amssymb}
\usepackage{algorithm}
\usepackage{subfig}
\usepackage{xcolor}
\usepackage[english]{babel}
\usepackage{graphicx}
\usepackage{grffile}
\usepackage{wrapfig,epsfig}
\usepackage{epstopdf}
\usepackage{url}
\usepackage{epstopdf}
\usepackage{algpseudocode}
\usepackage[T1]{fontenc}
\usepackage{bbm}
\usepackage{comment}
\usepackage{dsfont}

\usepackage{enumitem}

\usepackage{tikz}

\usetikzlibrary{arrows}

\ifdefined\isarxiv
\usepackage{hyperref} 
\hypersetup{colorlinks=true,citecolor=red,linkcolor=blue} 
\usepackage[margin=1in]{geometry}
\else
\usepackage{hyperref}   
\hypersetup{colorlinks=true,citecolor=red,linkcolor=blue} 
\fi
\graphicspath{{./figs/}}

\newtheorem{theorem}{Theorem}[section]
\newtheorem{lemma}[theorem]{Lemma}
\newtheorem{definition}[theorem]{Definition}

\newtheorem{proposition}[theorem]{Proposition}
\newtheorem{corollary}[theorem]{Corollary}

\newtheorem{fact}[theorem]{Fact}
\newtheorem{remark}[theorem]{Remark}
\newtheorem{claim}[theorem]{Claim}

\newcommand{\wt}{\widetilde}
\newcommand{\ov}{\overline}

\newcommand{\N}{\mathcal{N}}
\newcommand{\R}{\mathbb{R}}

\renewcommand{\d}{\mathrm{d}}

\renewcommand{\varepsilon}{\epsilon}

\newcommand{\err}{\mathrm{err}}

\DeclareMathOperator*{\E}{{\mathbb{E}}}

\DeclareMathOperator{\poly}{poly}

\DeclareMathOperator{\nnz}{nnz}

\DeclareMathOperator{\dis}{dis}
\DeclareMathOperator{\cts}{cts}
\DeclareMathOperator{\sgn}{sgn}
\DeclareMathOperator{\fire}{fire}
\DeclareMathOperator{\diag}{\mathrm{diag}}

\newcommand{\ket}[1]{\left| #1 \right\rangle}

\newcommand{\Tinit}{{\cal T}_{\mathsf{init}}}
\newcommand{\Tquery}{{\cal T}_{\mathsf{query}}}
\newcommand{\Tupdate}{{\cal T}_{\mathsf{update}}}
\renewcommand{\bar}{\overline}

\newcommand{\Zhao}[1]{{\color{red}[Zhao: #1]}}
\newcommand{\Ruizhe}[1]{{\color{blue}[Ruizhe: #1]}}

\newcommand{\blue}[1]{{\color{blue}#1}}
\makeatletter
\newcommand*{\RN}[1]{\expandafter\@slowromancap\romannumeral #1@}
\makeatother

\usepackage{lineno}

\begin{document}
\title{Does Preprocessing Help Training Over-parameterized Neural Networks?\thanks{A Preliminary version of this paper appeared in Conference on Neural Information Processing Systems (NeurIPS 2021)}}
\ifdefined\isarxiv

\date{}


\author{
Zhao Song\thanks{\texttt{zsong@adobe.com}. Adobe Research.}
\and
Shuo Yang\thanks{\texttt{yangshuo\_ut@utexas.edu}. The University of Texas at Austin. Supported by NSF grants 1564000 and 1934932.}
\and
Ruizhe Zhang\thanks{\texttt{ruizhe@utexas.edu}. The University of Texas at Austin. Supported by NSF Grant CCF-1648712 and Scott Aaronson’s Vannevar Bush Faculty Fellowship from the US Department of Defense.}
}

\else
\author{}

\fi

\ifdefined\isarxiv
\begin{titlepage}
  \maketitle
  \begin{abstract}
Deep neural networks have achieved impressive performance in many areas. Designing a fast and provable method for training neural networks is a fundamental question in machine learning. 

The classical training method requires paying $\Omega(mnd)$ cost for both forward computation and backward computation, where $m$ is the width of the neural network, and we are given $n$ training points in $d$-dimensional space. In this paper, we propose two novel preprocessing ideas to bypass this $\Omega(mnd)$ barrier:
\begin{itemize}
    \item First, by preprocessing the initial weights of the neural networks, we can train the neural network in $\wt{O}(m^{1-\Theta(1/d)} n d)$ cost per iteration.
    \item Second, by preprocessing the input data points, we can train the neural network in $\wt{O} (m^{4/5} nd )$ cost per iteration.
\end{itemize}

From the technical perspective, our result is a sophisticated combination of tools in different fields, greedy-type convergence analysis in optimization, sparsity observation in practical work, high-dimensional geometric search in data structure, concentration and anti-concentration in probability. Our results also provide theoretical insights for a large number of previously established fast training methods.

In addition, our classical algorithm can be generalized to the Quantum computation model. Interestingly, we can get a similar sublinear cost per iteration but avoid preprocessing initial weights or input data points.


  \end{abstract}
  \thispagestyle{empty}
\end{titlepage}

{
\small
\tableofcontents
}
\newpage

\else
\maketitle
\begin{abstract}
Deep neural networks have achieved impressive performance in many areas. Designing a fast and provable method for training neural networks is a fundamental question in machine learning. 

The classical training method requires paying $\Omega(mnd)$ cost for both forward computation and backward computation, where $m$ is the width of the neural network, and we are given $n$ training points in $d$-dimensional space. In this paper, we propose two novel preprocessing ideas to bypass this $\Omega(mnd)$ barrier:
\begin{itemize}
    \item First, by preprocessing the initial weights of the neural networks, we can train the neural network in $\wt{O}(m^{1-\Theta(1/d)} n d)$ cost per iteration.
    \item Second, by preprocessing the input data points, we can train the neural network in $\wt{O} (m^{4/5} nd )$ cost per iteration.
\end{itemize}

From the technical perspective, our result is a sophisticated combination of tools in different fields, greedy-type convergence analysis in optimization, sparsity observation in practical work, high-dimensional geometric search in data structure, concentration and anti-concentration in probability. Our results also provide theoretical insights for a large number of previously established fast training methods.

In addition, our classical algorithm can be generalized to the Quantum computation model. Interestingly, we can get a similar sublinear cost per iteration but avoid preprocessing initial weights or input data points.


\end{abstract}

\fi

\section{Introduction}




Over the last decade, deep learning has achieved dominating performance over many areas, e.g., computer vision \cite{lbbh98,ksh12,slj+15,hzrs16}, natural language processing \cite{cwb+11,dclt18}, game playing \cite{alphago16,alphago17} and beyond. The computational resource requirement for deep neural network training grows very quickly. Designing a fast and provable training method for neural networks is, therefore, a fundamental and demanding challenge.

Almost all deep learning models are optimized by gradient descent (or its variants). The total training time can be split into two components, the first one is the number of iterations and the second one is the cost per spent per iteration. 
Nearly all the iterative algorithms for acceleration can be viewed as two separate lines of research correspondingly, the first line is aiming for an algorithm that has as small as possible number of iterations, the second line is focusing on designing as efficient as possible data structures to improve the cost spent per iteration of the algorithm \cite{v89_lp,cls19,lsz19,jlsw20,jklps20,jswz21}. In this paper, our major focus is on the second line.

There are a number of practical works trying to use a nearest neighbor search data structure to speed up the per-step computation of the deep neural network training \cite{cmf+20,lxj+20,clp+21,dmzs21}. However, none of the previous work is able to give a provable guarantee. In this paper, our goal is to develop training algorithms that provably reduce per step time complexity. Let us consider the ReLU activation neural network and two-layer neural network\footnote{An alternative name of the two-layer neural network is ``one-hidden layer neural network''.}. Let $n$ denote the number of training data points. Let $d$ denote the dimension of each data point. Let $m$ denote the number of neurons. 
In each iteration of gradient descent (GD), we need to compute prediction for each point in the neural network. Each point $x_i \in \R^d$, requires to compute $m$ inner product in $d$ dimension. Thus, $\Omega(mnd)$ is a natural barrier for cost per iteration in training neural networks (in both forward computation and backward computation).

A natural question to ask is
\begin{center}
{\it
    Is it possible to improve the cost per iteration of training neural network algorithm? E.g., is $o(mnd)$ possible?
}
\end{center}

We list our contributions as follows: 
\begin{itemize}[leftmargin=*, itemsep=0.5ex]
    \item We provide a new theoretical framework for speeding up neural network training by: 1) adopting the shifted neural tangent kernel; 2) showing that only a small fraction ($o(m)$) of neurons are activated for each input data in each training iteration; 3) identifying the sparsely activated neurons via geometric search; 4) proving that the algorithm can minimize the training loss to zero in a linear convergence rate.
    \item We provide two theoretical results 1) our first result (Theorem~\ref{thm:main_informal_1}) builds a dynamic half-space report data structure for the weights of a neural network, to train neural networks in sublinear cost per iteration; 2) our second result (Theorem~\ref{thm:main_informal_2}) builds a static half-space report data-structure for the input data points of the training data set for training a neural network in sublinear time. 
\end{itemize}


\paragraph{Acceleration via high-dimensional search data-structure.}

High-dimensional search data structures support efficiently finding points in some geometric query regions (e.g., half-spaces, simplices, etc). Currently, there are two main approaches: one is based on Locality Sensitive Hashing (LSH) \cite{im98}, which aims to find the close-by points (i.e., small $\ell_2$ distance \cite{diim04,ar15,ailrs15,arm17,r17,air18,biw19,dirw20} or large inner product \cite{sl14,sl15_uai,sl15_www}) of a query $q\in \mathbb{R}^d$ in a given set of points $S\subset \mathbb{R}^d$. This kind of algorithms runs very fast in practice, but most of them only support approximate queries. Another approach is based on space partitioning data structures, for example, partition trees \cite{m92,m92b,aem92,ac09,cha12}, $k$-$d$ trees / range trees \cite{c17,tog17,c19}, Voronoi diagrams \cite{adms98, c00}, which can exactly search the query regions. Recent works have successfully applied high-dimensional geometric data structure to reduce the complexity of training deep learning models. SLIDE \cite{cmf+20} accelerates the forward pass by retrieving neurons with maximum inner product via an LSH-based data structure; Reforemer \cite{kkl20} similarly adopts LSH to reduce the memory usage for processing long sequence; MONGOOSE \cite{clp+21} accelerates the forward pass by retrieving neurons with maximum inner products via a learnable LSH-based data structure \cite{c02} and lazy update framework \cite{cls19}. Despite the great empirical success, there is no theoretical understanding of such acceleration.

The goal of our paper is to theoretically characterize the acceleration brought by the high-dimensional geometric data structure. Specifically, our algorithm and analysis are built upon the HSR data structures \cite{aem92} which can find all the points that have large inner products and support efficient data update. Note that HSR comes with a stronger recovery guarantee than LSH, in the sense that HSR, whereas LSH is guaranteed to find some of those points.

\paragraph{Convergence via over-parameterization.} Over the last few years, there has been a tremendous work studying the convergence result of deep neural network explicilty or implicitly based on neural tangent kernel (NTK) \cite{jgh18}, e.g.~\cite{ll18,dzps19,als19_dnn,als19_rnn,dllwz19,adhlw19,adhlsw19,sy19,cgh+19,zmg19,cg19,zg19,os19,lsswy20,jt20,zpdlsa20,hlsy21,bpsw21}. It has been shown that (S)GD can train a sufficiently wide NN  with random initialization will converge to a small training error in polynomial steps.


\section{Challenges and Techniques}\label{sec:techniques}

\begin{itemize}[leftmargin=*, itemsep=0.5ex]
    \item Empirical works combine high-dimensional search data structures (e.g., LSH) with neural network training, however, they do not work theoretically due to the following reasons:
    \begin{itemize}
        \item Without shifting, the number of activated (and therefore updated) neurons is $\Theta(m)$. There is no hope to theoretically prove $o(m)$ complexity (See \textbf{Challenge 1}).
        \item Approximate high-dimensional search data structures might miss some important neurons, which can potentially prevent the training from converging (see \textbf{Challenge 2}).
    \end{itemize}
    \item Our solutions are:
    \begin{itemize}
    \item We propose a shifted ReLU activation that is guaranteed to have $o(m)$ number of activated neurons. Along with the shifted ReLU, we also propose a shifted NTK to rigorously provide a convergence guarantee (see \textbf{Solution 1}).
    \item We adopt an exact high-dimensional search data structure that better couples with the shifted NTK. It takes $o(m)$ time to identify the activated neurons and fits well with the convergence analysis as it avoids missing important neurons (see \textbf{Solution 2}).
    \end{itemize}
\end{itemize}

\paragraph{Challenge 1: How to sparsify an over-parameterized neural network?}
To speed up the training process, we need the neural network to be ``sparse'', that is, for each training data $x\in \R^d$, the number of activated neurons is small. Then, in the forward computation, we can just evaluate a small subset of neurons. However, in the previous NTK analysis (e.g., \cite{dzps19}), the activation function is $\sigma(x)=\max\{\langle w_r, x\rangle, 0\}$, and the weights vectors $w_r$ are initially sampled from a standard $d$-dimensional Gaussian distribution. Then, by the symmetry of Gaussian distribution, we know that for every input data $x$, there will be about half of the neurons being activated, which means that we can only obtain a constant-factor speedup. 

\paragraph{Solution 1}
The problem actually comes from the activation function. In practice, people use a shifted ReLU function $\sigma_b(x) = \max\{\langle w_r, x\rangle, b_r\}$ to train neural networks. The main observation of our work is that \emph{threshold implies sparsity}. We consider the setting where all neurons have a unified threshold parameter $b$. Then, by the concentration of Gaussian distribution, there will be $O(\exp(-b^2)\cdot m)$ activated neurons after the initialization. 

The next step is to show that the number of activated neurons will not blow up too much in the following training iterations. In \cite{dzps19,sy19}, they showed that the weights vectors are changing slowly during the training process. In our work, we open the black box of their proof and show a similar phenomenon for the shifted ReLU function. More specifically, a key component is to prove that for each training data, a large fraction of neurons will not change their status (from non-activated to activated and vice versa) in the next iteration with high probability. To achieve this, they showed that this is equivalent to the event that a standard Gaussian random variable in a small centered interval $[-R, R]$, and applied the anti-concentration inequality to upper-bound the probability. In our setting, we need to upper-bound the probability of $z\sim {\cal N}(0,1)$ in a shifted interval $[b-R, b + R]$. On the one hand, we can still apply the anti-concentration inequality by showing that the probability is at most $\Pr[z\in [-R, R]]$. On the other hand, this probability is also upper-bounded by $\Pr[z>b-R]$, and for small $R$, we can apply the concentration inequality for a more accurate estimation. In the end, by some finer analysis of the probability, we can show that with high probability, the number of activated neurons in each iteration is also   $O(\exp(-b^2)\cdot m)$ for each training data. If we take $b=\Theta( \sqrt{\log m} )$, we only need to deal with truly sublinear in $m$ of activated neurons in the forward evaluation.


\paragraph{Challenge 2: How to find the small subset of activated neurons?} A linear scan of the neurons will lead to a time complexity linear in $m$, which we hope to avoid. Randomly sampling or using LSH for searching can potentially miss important neurons which are important for a rigorous convergence analysis.

\paragraph{Solution 2} Given the shifted ReLU function $\sigma_b( \langle w_r , x \rangle )=\max\{ \langle w_r , x \rangle - b, 0\}$,  
the active neurons are those with weights $w_r$ lying in the half space of $\langle w_r, x\rangle - b > 0$. Finding such neurons is equivalent to a computational geometry problem: given $m$ points in $\R^d$, in each query and a half space ${\cal H}$, the goal is to output the points contained in ${\cal H}$. Here we use the Half-Space Reporting (HSR) data structure proposed by \cite{aem92}: after proper initialization, the HSR data structure can return all points lying in the queried half space with complexity as low as $O(\log(n) + k)$, where $k$ is the number of such points. Note that the HSR data structure well couples with the shifted ReLU, as the number of activated neurons $k$ is truly sublinear in $m$ as per the setting of $b = \Theta(\sqrt{\log m})$.

\section{Preliminaries}
\paragraph{Notations}
For an integer $n$, we use $[n]$ to denote the set $\{1,2,\cdots,n\}$. For a vector $x$ and $p \in \{0,1,2,\infty\}$, we use $\| x \|_p$ to denote the entry-wise $\ell_p$ norm of a vector.    We use $I_d$ to denote $d$-dimensional identity matrix. We use ${\cal N}(\mu,\sigma^2)$ to denote Gaussian distribution with mean $mu$ and variance $\sigma^2$. We use $\wt{O}$ to hide the polylog factors.


This section is organized as follows.
 Section~\ref{sec:problem}  introduces the neural network and present problem formulation.   Section~\ref{sec:data_structure_hsr}  presents the half-space report data-structure, Section~\ref{sec:quantum_benefit_assumptions} proposes our new sparsity-based Characterizations.

\subsection{Problem Formulation}\label{sec:problem}

In this section, we introduce the neural network model we study in this work. Let us consider a two-layer ReLU activated neural network $f$ that has width $m$ and $\ell_2$ loss function. \footnote{This is a very standard formulation in the literature, e.g., see \cite{dzps19,sy19,bpsw21} 
} 
\begin{definition}[Prediction function and loss function]\label{def:2nn}
Given $b\in \R$, $x \in \R^d$, $W \in \R^{d \times m}$ and $a \in \R^m$,
\begin{align*}
f (W,x,a) := & ~ \frac{1}{ \sqrt{m} } \sum_{r=1}^m a_r \sigma_b ( \langle w_r , x \rangle ) , \\
L(W) := & ~ \frac{1}{2} \sum_{i=1}^n ( f (W,x_i,a) - y_i )^2.
\end{align*}
We say function $f$ is $\mathsf{2NN}(m,b)$ for simplicity.
\end{definition}

 Here $W$ are weights that connect input nodes with hidden nodes, $a_1, \cdots, a_m \in \R$ are the weights that connect hidden nodes with output node. The ReLU function $\sigma_b(x):=\max\{x-b, 0\}$, where $b$ is the threshold parameter. Following the literature, we mainly focus on optimizing $W \in \R^{d \times m}$. For weights $a \in \R^m$, we will never change $a$ during the training after we randomly choose them at the initialization.\footnote{We remark, in some previous work, they do choose shift, but their shift is a random shift. In our application, it is important that the same $b$ is fixed for all neurons and never trained. } 
\begin{definition}[Weights at initialization]\label{def:init}
We use the following initialization,
\begin{itemize}
    \item For each $r$, we sample $ w_r(0) \sim \N(0,I_d)$
    \item For each $r$, we sample $a_r$ from $\{-1,+1\}$ uniformly at random
\end{itemize}
\end{definition}

Next, we can calculate the gradient
\begin{fact}[Gradient of the prediction function and loss function]
For each $r \in [m]$,
\begin{align}\label{eq:relu_derivative}
\frac{\partial f (W,x,a)}{\partial w_r}=\frac{1}{ \sqrt{m} } a_r x{\bf 1}_{ w_r^\top x \geq b }.
\end{align}
and 
\begin{align}\label{eq:gradient}
\frac{ \partial L(W) }{ \partial w_r } = \frac{1}{ \sqrt{m} } \sum_{i=1}^n ( f(W,x_i,a) - y_i ) a_r x_i {\bf 1}_{ \langle w_r , x_i \rangle \geq b }.
\end{align}
\end{fact}


To update the weights from iteration $k$ to iteration $k+1$, we follow the standard update rule of the GD algorithm,
\begin{align}\label{eq:w_update}
\text{~GD:~~~} W(k+1) = W(k) - \eta \cdot \Delta W(k), \text{~~~where~~~} \Delta W(k) = \frac{ \partial L( W(k) ) }{ \partial W(k) } .
\end{align}

The ODE of the gradient flow is defined as
\begin{align}\label{eq:wr_derivative}
\frac{\d w_r(t)}{\d t}=-\frac{ \partial L(W) }{ \partial w_r }.
\end{align}

\begin{definition}[Error of prediction]\label{def:err}
For each $t \in \{0,1,\cdots,T\}$,
we define $\err(t) \in \R^n$ to be the error of prediction $\err(t) = y- u(t) $, where $u(t):=f(W(t), a, X)\in \mathbb{R}^n$
\end{definition}

\subsection{Data Structure for Half-Space Reporting}\label{sec:data_structure_hsr}

The half-space range reporting problem is an important problem in computational geometry, which is formally defined as following:
\begin{definition}[Half-space range reporting]\label{def:HSR}
Given a set $S$ of $n$ points in $\R^d$. There are two operations:
\begin{itemize}
    \item \textsc{Query}$(H)$: given a half-space $H \subset \R^d$, output all of the points in $S$ that contain in $H$, i.e., $S\cap H$. 
    \item \textsc{Update}: add or delete a point in $S$. 
    \begin{itemize}
        \item \textsc{Insert}$(q)$: insert $q$ into $S$
        \item \textsc{Delete}$(q)$: delete $q$ from $S$
    \end{itemize}
\end{itemize}
Let $\Tinit$ denote the pre-processing time to build the data structure, $\Tquery$ denote the time per query and $\Tupdate$ time per update.
\end{definition}

We use the data-structure proposed in \cite{aem92} to solve the half-space range reporting problem, which admits the interface summarized in Algorithm~\ref{alg:half_space_report}. Intuitively, the data-structure recursively partitions the set $S$ and organizes the points in a tree data-structure. Then for a given query $(a, b)$, all $k$ points of $S$ with $\sgn(\langle a, x\rangle -b)\geq 0$ are reported quickly. Note that the query $(a, b)$ here defines the half-space $H$ in Definition~\ref{def:HSR}.

\begin{algorithm}[H]
\caption{Half Space Report Data Structure}\label{alg:half_space_report}
\begin{algorithmic}[1]
    \algrenewcommand\algorithmicprocedure{\textbf{data structure}}
    \Procedure{HalfSpaceReport}{}
        \State {\bf procedures:}
        \State \hspace{4mm} \textsc{Init}($S, n, d$) \Comment{Initialize the data structure with a set $S$ of $n$ points in $\R^d$}
        \State \hspace{4mm} \textsc{Query}($a,b$)\Comment{$a,b\in \R^d$. Output the set $\{x\in S: \sgn(\langle a, x\rangle -b)\geq 0\}$}
        \State \hspace{4mm} \textsc{Add}($x$)\Comment{Add point $x\in \R^d$ to $S$}
        \State \hspace{4mm} \textsc{Delete}($x$)\Comment{Delete point $x\in \R^d$ from $S$}
    \EndProcedure
\end{algorithmic}
\end{algorithm}

Adapted from \cite{aem92}, the algorithm comes with the following complexity:

\begin{corollary}[\cite{aem92}]\label{cor:aem92}
Given a set of $n$ points in $\R^d$, the half-space reporting problem can be solved with the following performances:
\begin{itemize}
    \item Part 1. 
    $\Tquery(n,d,k)=O_{d}(n^{1-1/\lfloor d/2\rfloor}+k)$, amortized $\Tupdate=O_{d}(\log^2(n))$.
    \item Part 2. 
    $\Tquery(n,d,k)=O_{d}(\log(n)+k)$, amortized $\Tupdate=O_{d}(n^{\lfloor d/2\rfloor-1 })$.
\end{itemize}
\end{corollary}
We remark that Part 1 will be used in Theorem~\ref{thm:main_informal_1} and Part 2 will be used in Theorem~\ref{thm:main_informal_2}.

\subsection{Sparsity-based   Characterizations}\label{sec:quantum_benefit_assumptions}


In this section, we consider the ReLU function with a nonzero threshold: $\sigma_b(x) = \max\{0, x-b\}$, which is commonly seen in practise, and also has been considered in theoretical work \cite{zpdlsa20}.

We first define the set of neurons that are firing at time $t$. 
\begin{definition}[fire set]
For each $i \in [n]$, for each $t \in \{0,1,\cdots,T\}$, let ${\cal S}_{i,\mathrm{fire}}(t) \subset [m]$ denote the set of neurons that are ``fire'' at time $t$, i.e.,
\begin{align*}
{\cal S}_{i,\mathrm{fire}}(t) := \{ r \in [m] : \langle w_r(t), x_i \rangle  > b   \}    .
\end{align*}
We define $k_{i,t} := |{\cal S}_{i,\mathrm{fire}}(t) |$, for all $ t$ in $ \{0,1,\cdots,T\}$.
\end{definition}






We propose a new ``sparsity'' lemma in this work. It shows that $\sigma_b$ gives the desired sparsity.

\begin{lemma}[Sparsity after initialization]\label{lem:sparse_initial}
Let $b> 0$ be a tunable parameter. If we use the $\sigma_b$ as the activation function, then after the initialization, with probability at least $1-n\cdot \exp(-\Omega(m \cdot \exp(-b^2/2)))$, it holds that for each input data $x_i$, the number of activated neurons $k_{i,0}$ is at most $ O(m\cdot \exp(-b^2/2))$, where $m$ is the total number of neurons.
\end{lemma}

\begin{proof}
By the concentration of Gaussian distribution, the initial fire probability of a single neuron is
\begin{align*}
    \Pr[\sigma_b(\langle w_r(0), x_i\rangle)>0] = &~ \Pr_{z\sim \mathcal{N}(0, 1)}[z>b]\leq \exp(-b^2/2).
\end{align*}
Hence, for the indicator variable $\mathbf{1}_{r\in {\cal S}_{i,\fire}(0)}$, we have
\begin{align*}
    \E[\mathbf{1}_{r\in {\cal S}_{i,\fire}(0)}] \leq \exp(-b^2/2).
\end{align*}
By standard concentration inequality (Lemma \ref{lem:bernstein}),
\begin{align}\label{eq:s_0_bernstein}
\Pr \left[ |{\cal S}_{i,\fire}(0)| > k_0+t \right] \leq \exp \left(-\frac{t^2/2}{k_0+t/3} \right), \forall t >0
\end{align}
where $k_0 := m \cdot \exp(-b^2/2)$. If we choose $t=k_0$, then we have:
\begin{align*}
    \Pr \left[ |{\cal S}_{i,\fire}(0)| > 2k_0 \right] \leq \exp \left(-3k_0 / 8\right)
\end{align*}

Then, by union bound over all $i\in [n]$, we have that with high probability 
\begin{align*}
1 - n \cdot \exp(-\Omega(m \cdot \exp(-b^2/2))) ,
\end{align*}
the number of initial fire neurons for the sample $x_i$ is bounded by $k_{i, 0}\leq 2m\cdot \exp(-b^2/2)$.
\end{proof}

The following remark gives an example of setting the threshold $b$, and will be useful for showing the sublinear complexity in the next section.
\begin{remark}\label{remark:sparse}
If we choose $b = \sqrt{0.4\log m}$ then $k_0 = m^{4/5}$. For $t= m^{4/5}$, Eq.~\eqref{eq:s_0_bernstein} implies that
\begin{align*}
    \Pr \left[ |{\cal S}_{i,\fire}(0)| > 2m^{4/5} \right] \leq \exp \left(- \min \{ m R, O(m^{4/5} ) \} \right) .
\end{align*}
\end{remark}


\section{Training Neural Network with Half-Space Reporting Data Structure}

In this section, we present two sublinear time  algorithms for training over-parameterized neural networks. The first algorithm (Section~\ref{subsec:weights_preprocessing}) relies on building a high-dimensional search data-structure for the weights of neural network. The second algorithm (Section~\ref{subsec:data_preprocessing}) is based on building a data structure for the input data points of the training set. Both of the algorithms use the HSR to quickly identify the fired neurons to avoid unnecessary calculation. The time complexity and the sketch of the proof are provided after each of the algorithms.

\subsection{Weights Preprocessing}\label{subsec:weights_preprocessing}

We first introduce the algorithm that preprocesses the weights $w_r$ for $r \in [m]$, which is commonly used in practice \cite{clp+21,cmf+20,kkl20}. Recall $\mathsf{2NN}(m,b)$ is 
$f (W,x,a) := \frac{1}{ \sqrt{m} } \sum_{r=1}^m a_r \sigma_b ( \langle w_r , x \rangle )$. By constructing a HSR data-structure for $w_r$'s, we can quickly find the set of active neurons $S_{i, \fire}$ for each of the training sample $x_i$. See pseudo-code in Algorithm~\ref{alg:ds_for_w_training}. 


\begin{algorithm}[t]
\caption{Training Neural Network via building a data structure of weights of the neural network}
\label{alg:ds_for_w_training} 
\begin{algorithmic}[1]
    \algrenewcommand\algorithmicprocedure{\textbf{procedure}}
	\Procedure{TrainingWithPreprocessWeights}{$\{(x_i,y_i)\}_{i\in [n]}$,$n$,$m$,$d$} \Comment{Theorem~\ref{thm:main_informal_1}}
	\State Initialize $w_r, a_r$ for $r\in [m]$ and $b$ according to Definition~\ref{def:init} and Remark~\ref{remark:sparse}
    \State \textsc{HalfSpaceReport} \textsc{hsr}.\textsc{Init}($\{w_r(0)\}_{r\in [m]}, m, d$) 
    \Comment{Algorithm~\ref{alg:half_space_report}}

	\For{$t=1 \to T$}

	        \State $S_{i,\fire}\gets \textsc{hsr}.\textsc{Query}(x_i,b)$ for $i \in [n]$ 

		    \State Forward pass for $x_i$ only on neurons in $S_{i,\fire}$ for $i \in [n]$
		    \State Calculate gradient for $x_i$ only on neurons in $S_{i, \fire}$ for $i \in [n]$

            \State Gradient update for the neurons in $\cup_{i \in [n]}S_{i, \fire}$

		    \State \textsc{hsr}.\textsc{Delete}$(w_r(t))$ for $r \in \cup_{i\in [n]} S_{i, \fire}$ \label{ln:del_w}
		    \State \textsc{hsr}.\textsc{Add}$(w_r(t+1))$ for $r \in \cup_{i\in [n]} S_{i, \fire}$ \label{ln:add_w}
	\EndFor
	\State \Return Trained weights $w_r(T + 1)$ for $r \in [m]$
	\EndProcedure
\end{algorithmic}
\end{algorithm}

In the remaining part of this section, we focus on the time complexity analysis of Algorithm~\ref{alg:ds_for_w_training}. The convergence proof will be given in Section~\ref{sec:quartic_suffices}.



\begin{lemma}[Running time part of Theorem~\ref{thm:main_informal_1}]
Given $n$ data points in $d$-dimensional space. Running gradient descent algorithm (Algorithm~\ref{alg:ds_for_w_training}) on $\mathsf{2NN}(m,b=\sqrt{0.4 \log m})$ (Definition~\ref{def:2nn})  
the expected cost per-iteration  of the gradient descent algorithm is 
\begin{align*}
     \wt{O} (m^{1-\Theta(1/d)} nd ).
\end{align*}
\end{lemma}
\begin{proof}
  The per-step time complexity is
  \begin{align*}
      \sum_{i=1}^n {\cal T}_{\textsc{Query}}(m,d,k_{i,t})  + ({\cal T}_{\textsc{Delete}}+ {\cal T}_{\textsc{Insert}})\cdot|\cup_{i\in [n]} S_{i, \fire}(t)| + d \sum_{i\in [n]}k_{i, t}
  \end{align*}
  The first term $\sum_{i=1}^n{\cal T}_{\textsc{Query}}(m,d,k_{i,t})$ corresponds to the running time of querying the active neuron set $S_{i, \fire}(t)$ for all training samples $i \in [n]$. With the first result in Corollary~\ref{cor:aem92}, the complexity is bounded by $\wt O(m^{1 - \Theta(1/d)}nd)$.
  
  The second term $({\cal T}_{\textsc{Delete}}+ {\cal T}_{\textsc{Insert}})\cdot |\cup_{i\in [n]} S_{i, \fire}(t)|$ corresponds to updating $w_r$ in the high-dimensional search data-structure (Lines~\ref{ln:del_w} and \ref{ln:add_w}). Again with the first result in Corollary~\ref{cor:aem92}, we have ${\cal T}_{\textsc{Delete}}+ {\cal T}_{\textsc{Insert}} = O(\log^2 m)$. Combining with the fact that $|\cup_{i\in [n]} S_{i, \fire}(t)| \le |\cup_{i\in [n]} S_{i, \fire}(0)| \le O(n m^{4/5})$, the second term is bounded by $O(n m^{4/5} \log^2 m)$. 
  
  The third term is the time complexity of gradient calculation restricted to the set ${\cal S}_{i, \fire}(t)$. With the bound on $\sum_{i \in [n]}k_{i, t}$ (Lemma~\ref{lem:bound_fire_neurons}), we have $d\sum_{i\in [n]}k_{i,t}\leq O( m^{4/5} nd)$. 
  
  Putting them together completes the proof.
\end{proof}

\subsection{Data Preprocessing}\label{subsec:data_preprocessing}

\begin{algorithm}[t]
\caption{Training Neural Network via building a data-structure of the input data points}
\label{alg:ds_for_x_training} 
\begin{algorithmic}[1]
    \algrenewcommand\algorithmicprocedure{\textbf{procedure}}
	\Procedure{TrainingWithProcessData}{$\{(x_i,y_i)\}_{i\in [n]}$,$n$,$m$,$d$} \Comment{Theorem~\ref{thm:main_informal_2}}
	\State Initialize $w_r, a_r$ for $r\in [m]$ and $b$ according to Definition~\ref{def:init} and Remark~\ref{remark:sparse}
    \State \textsc{HalfSpaceReport} \textsc{hsr}.\textsc{Init}($\{x_i\}_{i \in [n]}, n, d$) \Comment{Algorithm~\ref{alg:half_space_report}}
        \State $\wt{S}_{r,\fire} \leftarrow \textsc{hsr}.\textsc{Query}(w_r(0),b)$ for $r\in [m]$  \Comment{$\wt{S}_{r,\fire}$ are samples which neuron $r$ fires for} \label{ln:init_s_tilde}
        \State $S_{i, \fire} \leftarrow \{ r~|~i \in \wt{S}_{r, \fire} \}$ \Comment{$S_{i,\fire}$ is the set of neurons, which fire for $x_i$} \label{ln:init_s}
	\For{$t=1 \to T$}
	   \State Forward pass for $x_i$ only on neurons in $S_{i,\fire}$ for $i \in [n]$ \label{ln:forward_b}
		    \State Calculate gradient for $x_i$ only on neurons in $S_{i, \fire}$ for $i \in [n]$
            \State Gradient update for the neurons in $\cup_{i \in [n]}S_{i, \fire}$  \label{ln:backward_e}
		\For{$r \in \cup_{i \in [n]}\mathcal{S}_{i, \fire}$} \label{ln:maintain_b}
		
		\State $S_{i, \fire}.\textsc{Del}(r)$ for $i \in \wt{S}_{r,\fire}$ 
        \State $\wt{S}_{r,\fire} \leftarrow \textsc{hsr}.\textsc{Query}(w_r(t+1),b)$  
        \State $S_{i, \fire}.\textsc{Add}(r)$ for $i \in \wt{S}_{r,\fire}$ 
    \EndFor\label{ln:maintain_e}
	\EndFor
	\State \Return Trained weights $w_r(T + 1)$ for $r \in [m]$
	\EndProcedure
\end{algorithmic}
\end{algorithm}

While the weights preprcessing algorithm is inspired by the common practise, the dual relationship between the input $x_i$ and model weights $w_r$ inspires us to preprocess the dataset before training (i.e., building HSR data-structure for $x_i$). This largely improves the per-iteration complexity and avoids the frequent updates of the data structure since the training data is fixed. More importantly, once the training dataset is preprocessed, it can be reused for different models or tasks, thus one does not need to perform the expensive preprocessing for each training.

The corresponding pseudocode is presented in Algorithm~\ref{alg:ds_for_x_training}. With $x_i$ preprocessed, we can query HSR with weights $w_r$ and the result $\wt{S}_{r, \fire}$ is the set of training samples $x_i$ for which $w_r$ fires for. Given $\wt{S}_{r, \fire}$ for $r\in [m]$, we can easily reconstruct the set $S_{i, \fire}$, which is the set of neurons fired for sample $x_i$. The forward and backward pass can then proceed similar to Algorithm~\ref{alg:ds_for_w_training}.

At the end of each iteration, we will update $\wt{S}_{r, \fire}$ based on the new $w_r$ estimation and update $S_{i, \fire}$ accordingly. For Algorithm~\ref{alg:ds_for_x_training}, the HSR data-structure is static for the entire training process. This is the main difference from Algorithm~\ref{alg:ds_for_w_training}, where the HSR needs to be updated every time step to account for the changing weights $w_r$.

We defer the convergence analysis to Section~\ref{sec:quartic_suffices} and focus on the time complexity analysis of Algorithm~\ref{alg:ds_for_w_training} in the rest of this section. We consider $d$ being a constant for the rest of this subsection.

\begin{lemma}[Running time part of Theorem~\ref{thm:main_informal_2}]
Given $n$ data points in $d$-dimensional space. Running gradient descent algorithm (Algorithm~\ref{alg:ds_for_w_training}) on $\mathsf{2NN}(m,b=\sqrt{0.4 \log m})$ (Definition~\ref{def:2nn}), 
the expected per-iteration running time of initializing 
  $\wt{S}_{r, \fire}, S_{i, \fire}$ for $r\in [m], i\in[n]$ is
  $
      O(m\log n + m^{4/5} n).
  $
  The cost per iteration of the training algorithm is $
  O( m^{4/5} n\log n).
  $
\end{lemma}

\begin{proof}
We analyze the initialization and training parts separately.
\paragraph{Initialization}
 
In Lines~\ref{ln:init_s_tilde} and \ref{ln:init_s}, the sets $\wt{S}_{r, \fire}, S_{i, \fire}$ for $r\in [m], i\in[n]$ are initialized. For each $r\in [m]$, we need to query the data structure the set of data points $x$'s such that $\sigma_b(w_r(0)^\top x)>0$.  Hence, the running time of this step is
  \begin{align*}
      \sum_{r=1}^{m}\Tquery(n,d, \wt{k}_{r,0}) 
      = & ~ O(m \log n + \sum_{r=1}^m\wt{k}_{r, 0})  \\
      = & ~  O(m \log n + \sum_{i=1}^n k_{i, 0}) \\
      = & ~  O(m \log n + m^{4/5} n) .
  \end{align*}
  where the second step follows from $\sum_{r=1}^m \wt{k}_{r,0} = \sum_{i=1}^n k_{i,0}$.
  
\paragraph{Training}
Consider training the neural network for $T$ steps. For each step, first notice that the forward and backward computation parts (Line~\ref{ln:forward_b} - \ref{ln:backward_e}) are the same as previous algorithm. The time complexity is $O(m^{4/5} n\log n)$.
  
  We next show that maintaining $\wt{S}_{r, \fire}, r\in[m]$ and $S_{i, \fire}, i\in[n]$ (Line~\ref{ln:maintain_b} - \ref{ln:maintain_e}) takes  $O(m^{4/5} n\log n)$ time. For each fired neuron $r\in [m]$, we first remove the indices of data in the sets $S_{i,\mathsf{fire}}$, which takes time
  \begin{align*}
  O(1) \cdot \sum_{r \in \cup_{i \in [n]} S_{i, \fire}}\wt{k}_{r, t} = O(1) \cdot \sum_{r =1}^{m}\wt{k}_{r, t} = O(m^{4/5}  n).
  \end{align*}
  Then, we find the new set of $x$'s such that $\sigma_b( \langle w_r(t+1), x \rangle )>0$ by querying the half-space reporting data structure. The total running time for all fired neurons is
  \begin{align*}
      \sum_{r \in \cup_{i \in [n]} S_{i, \fire}}\Tquery(n,d, \wt{k}_{r,t+1}) \lesssim  m^{4/5} n \log n  + \sum_{r \in \cup_{i \in [n]} S_{i, \fire}}\wt{k}_{r, t+1}  = O(m^{4/5}  n \log n)
  \end{align*}
  Then, we update the index sets $S_{i, \fire}$ in time $O(m^{4/5}  n)$. Therefore, each training step takes $O(m^{4/5}  n \log n)$ time, which completes the proof.
\end{proof}





\section{Convergence of Our Algorithm }\label{sec:quartic_suffices}

We state the result of our training neural network algorithms (Lemma~\ref{lem:quartic}) can converge in certain steps. An important component in our proof is to find out a lower bound on minimum eigenvalue of the continuous Hessian matrix $\lambda_{\min}(H^{\cts})$. It turns out to be an anti-concentration problem of the Gaussian random matrix. 
In \cite{os19}, they gave a lower bound on $\lambda_{\min}(H^{\cts})$ for ReLU function with $b=0$, assuming the input data are separable. One of our major technical contribution is generalizing it to arbitrary $b \geq 0$.
\begin{proposition}[Informal version of Theorem~\ref{thm:sep}]\label{prop:sep}
Given $n$ (normalized) input data points $\{x_1,x_2, \cdots, x_n \} \subseteq \R^d$ such that $\forall i \in [n], \| x_i\|_2=1$. Let parameter $\delta:=\min_{i \neq j} \{ \| x_i - x_j \|_2, \| x_i + x_j \|_2 \}$ denote the data separability. For any shift parameter $b\geq 0$, we define shifted NTK $H^{\cts}\in \R^{n \times n}$ as follows
\begin{align*}
H^{\cts}_{i,j} :=   \E_{w \sim \N(0,I_d)} \left[ \langle x_i , x_j \rangle \cdot {\bf 1}_{ \langle w , x_i \rangle \geq b} \cdot {\bf 1}_{ \langle w, x_j \rangle \geq b } \right] , \forall i \in [n] , j \in [n] .
\end{align*}
Then
\begin{align*}
    \lambda_{\min} (H^{\cts}) \geq 0.01 e^{-b^2/2} \delta / n^2. 
\end{align*}

\end{proposition}

With proposition~\ref{prop:sep}, we are ready to show the convergence rate of training an over-parameterized neural network with shifted ReLU function. 


\begin{lemma}[Convergence part of Theorem~\ref{thm:main_informal_1} and Theorem~\ref{thm:main_informal_2}]\label{lem:quartic}
Suppose input data-points are $\delta$-separable, i.e., 
$
    \delta: = \min_{i \neq j} \{ \| x_i - x_j \|_2, \| x_i + x_j \|_2 \}
$. 
Let $m = \poly(n, 1/\delta, \log(n/\rho))$ and $\eta = O(\lambda/n^2)$. 
Let $b=\Theta(\sqrt{\log m})$. 
Then 
\begin{align*}
\Pr \Big[ 
\| \err(k) \|_2^2 \leq ( 1 - \eta \lambda / 2 )^k \cdot \| \err(0) \|_2^2, ~\forall k \in \{0,1,\cdots, T \} \Big] \geq 1-\rho.
\end{align*}
Note that the randomness is over initialization. Eventually, we choose $T = \lambda^{-2} n^2 \log(n/\epsilon)$ where $\epsilon$ is the final accuracy.
\end{lemma}
This result shows that despite the shifted ReLU and sparsely activated neurons, we can still retain the linear convergence. Combined with the results on per-step complexity in the previous section, it gives our main theoretical results of training deep learning models with sublinear time complexity (Theorem~\ref{thm:main_informal_1} and Theorem~\ref{thm:main_informal_2}).

\section{Main Classical Results}
We present two theorems (under classical computation model) of our work, showing the sublinear running time and linear convergence rate of our two algorithms. We leave the quantum application into Appendix~\ref{sec:quantum_alg}. 
The first algorithm is relying on building a high-dimensional geometric search data-structure for the weights of a neural network.
\begin{theorem}[Main result I, informal of Theorem~\ref{thm:main_1}]\label{thm:main_informal_1}
Given $n$ data points in $d$-dimensional space. We preprocess the initialization weights of the neural network.
Running gradient descent algorithm (Algorithm~\ref{alg:ds_for_w_training}) on a two-layer, $m$-width,  over-parameterized ReLU neural network will minimize the training loss to zero, 
and the expected running time of gradient descent algorithm (per iteration) is
\begin{align*}
\wt{O} (m^{1-\Theta(1/d)} nd ).
\end{align*}
\end{theorem}

The second algorithm is based on building a data structure for the input data points of the training set. Our second algorithm can further reduce the cost per iteration from $m^{1-1/d}$ to truly sublinear in $m$, e.g. $m^{4/5}$.
\begin{theorem}[Main result II, informal of Theorem~\ref{thm:main_1}]\label{thm:main_informal_2}
Given $n$ data points in $d$-dimensional space. We preprocess all the data points.
Running gradient descent algorithm (Algorithm~\ref{alg:ds_for_x_training}) on a two-layer, $m$-width,  over-parameterized ReLU neural network will minimize the training loss to zero, 
and the expected running time of gradient descent algorithm (per iteration) is
\begin{align*}
\wt{O} (m^{4/5} nd ).
\end{align*}
\end{theorem}

\section{Discussion and Limitations}\label{sec:discussion}

In this paper, we propose two sublinear algorithms to train neural networks. By preprocessing the weights of the neuron networks or preprocessing the training data, we rigorously prove that it is possible to train a neuron network with sublinear complexity, which overcomes the $\Omega(mnd)$ barrier in classical training methods. Our results also offer theoretical insights for many previously established fast training methods.

Our algorithm is intuitively related to the lottery tickets hypothesis \cite{frankle2018lottery}. However, our theoretical results can not be applied to explain lottery tickets immediately for two reasons: 1) the lottery ticket hypothesis focuses on pruning weights; while our results identify the important neurons. 2) the lottery ticket hypothesis identifies the weights that need to be pruned after training (by examining their magnitude), while our algorithms accelerate the training via preprocessing. It would be interesting to see how our theory can be extended to the lottery ticket hypothesis.

One limitation of our work is that the current analysis framework does not provide a convergence guarantee for combining LSH with gradient descent, which is commonly seen in many empirical works. Our proof breaks as LSH might miss important neurons which potentially ruins the convergence analysis. Instead, we refer to the HSR data structure, which provides a stronger theoretical guarantee of successfully finding all fired neurons.



\ifdefined\isarxiv

\newpage
\else

\fi

\ifdefined\isarxiv
\bibliographystyle{alpha}
\else
\bibliographystyle{alpha}
\fi
\bibliography{ref}
\newpage
\appendix 
\onecolumn

\paragraph{Roadmap.}
In Section~\ref{sec:complete_algorithm}, we present our main algorithms. In Section~\ref{sec:probability_tools}, we provide some preliminaries. In Section~\ref{sec:missing_proof}, we provide sparsity analysis. We show convergence analysis in Section~\ref{sec:converge_analysis}. In Section~\ref{sec:combine}, we show how to combine the sparsity, convergence, running time all together. In Section~\ref{sec:separation}, we show correlation between sparsity and spectral gap of Hessian in neural tangent kernel. In Section~\ref{sec:quantum_alg}, we discuss how to generalize our result to quantum setting.

\section{Complete Algorithms}\label{sec:complete_algorithm}

In this section, we present three algorithms (Alg.~\ref{alg:ds_full}, Alg.~\ref{alg:weight_full} and Alg.~\ref{alg:data_full}) which are the complete version of Alg.~\ref{alg:half_space_report}, Alg.~\ref{alg:ds_for_w_training} and Alg.~\ref{alg:ds_for_x_training}. 

\begin{algorithm}[H]
\caption{Half Space Report Data Structure}\label{alg:ds_full}
\begin{algorithmic}[1]
    \algrenewcommand\algorithmicprocedure{\textbf{data structure}}
    \Procedure{HalfSpaceReport}{}
        \State {\bf procedures:}
        \State \hspace{4mm} \textsc{Init}($S, n, d$) \Comment{Initialize the data structure with a set $S$ of $n$ points in $\R^d$}
        \State \hspace{4mm} \textsc{Query}($a,b$)\Comment{$a,b\in \R^d$. Output the set $\{x\in S: \sgn(\langle a, x\rangle -b)\geq 0\}$}
        \State \hspace{4mm} \textsc{Add}($x$)\Comment{Add a point $x\in \R^d$ to $S$}
        \State \hspace{4mm} \textsc{Delete}($x$)\Comment{Delete the point $x\in \R^d$ from $S$}
    \EndProcedure
\end{algorithmic}
\end{algorithm}

\begin{algorithm}[H]
\caption{Training Neural Network via building a data structure of weights.}\label{alg:weight_full}
\begin{algorithmic}[1]
    \algrenewcommand\algorithmicprocedure{\textbf{procedure}}
	\Procedure{TrainingWithPreprocessWeights}{$\{x_i\}_{i\in [n]}, \{y_i\}_{i\in [n]}$,$n$,$m$,$d$} \Comment{Theorem~\ref{thm:main_informal_1}}
	\State \blue{/*Initialization step*/}
	\State Sample $W(0)$ and $a$ according to Definition~\ref{def:init}
    \State $b\gets \sqrt{0.4\log m}$.
    \State \blue{/*A dynamic data-structure*/}
    \State \textsc{HalfSpaceReport} \textsc{hsr} \Comment{Algorithm~\ref{alg:half_space_report}, Part 1 of Corollary~\ref{cor:aem92}}
    \State \textsc{hsr}.\textsc{Init}($\{w_r(0)\}_{r\in [m]}, m, d$) \Comment{It takes $\Tinit(m,d)$ time}
    \State \blue{/*Iterative step*/}
	\For{$t=0 \to T$}
	    \State \blue{/*Forward computation step*/}
	    \For{$i=1 \to n$} \label{ln:classical_1}
	        \State $S_{i,\mathrm{fire}}\gets \textsc{hsr}.\textsc{Query}(x_i,b)$ \Comment{It takes $\Tquery(m,d, k_{i,t})$ time}
		    \State $u(t)_i \leftarrow \frac{1}{ \sqrt{m} } \sum_{r\in \mathcal{S}_{i,\mathrm{fire}}} a_r \cdot \sigma_b(w_r(t)^\top x_i)$ \Comment{It takes $O(d\cdot k_{i,t})$ time}
		\EndFor
		\State \blue{/*Backward computation step*/}
		\State $P \leftarrow 0^{n \times m}$ \Comment{$P \in \R^{n \times m}$}
		\For{$i = 1 \to n$} \label{ln:classical_2}
		    \For{$r \in {\cal S}_{i,\mathrm{fire}}$}
			    \State $P_{i,r} \leftarrow \frac{1}{\sqrt{m}} a_r \cdot \sigma_b'( w_r(t)^\top x_i )$ 
			\EndFor
		\EndFor
		\State $M\gets X\diag(y-u(t))$ \Comment{$M\in \R^{d\times n}$, it takes $O(n\cdot d)$ time}
		\State $\Delta W \leftarrow \underbrace{M}_{d\times n}\underbrace{P}_{n\times m}$\label{ln:compute_delta_1}\Comment{$\Delta W\in \R^{d\times m}$, it takes $O(d\cdot \nnz(P))$ time, $\nnz(P) = O(nm^{4/5})$}
		\State $W(t+1)\gets W(t)-\eta \cdot \Delta W$. 
		\State \blue{/*Update data structure*/}
		\State Let $Q \subset [m]$ where for each $r \in Q$, the $\Delta W_{*,r}$ is not all zeros \Comment{$|Q| \leq O(n m^{4/5})$}
		\For{$r \in Q$}
		    \State \textsc{hsr}.\textsc{Delete}$(w_r(t))$
		    \State \textsc{hsr}.\textsc{Insert}$(w_r(t+1))$
		\EndFor
	\EndFor
	\State \Return $W$ \Comment{$W \in \R^{d \times m}$}
	\EndProcedure
\end{algorithmic}
\end{algorithm}

\begin{algorithm}[H]
\caption{Training Neural Network via building a data-structure of the input points.}\label{alg:data_full}
\small
\begin{algorithmic}[1]
    \algrenewcommand\algorithmicprocedure{\textbf{procedure}}
	\Procedure{TrainingWithProcessData}{$\{x_i\}_{i\in[n]},\{y_i\}_{i\in [n]}$,$n$,$m$,$d$} \Comment{Theorem~\ref{thm:main_informal_2}}
	\State \blue{/*Initialization step*/}
	\State Sample $W(0)$ and $a$ according to Definition~\ref{def:init}
    \State $b\gets \sqrt{0.4\log m}$.
   \State \blue{/*A static data-structure*/}
    \State \textsc{HalfSpaceReport} \textsc{hsr} \Comment{Algorithm~\ref{alg:half_space_report}, Part 2 of Corollary~\ref{cor:aem92}}
    \State \textsc{hsr}.\textsc{Init}($\{x_i\}_{i \in [n]}, n, d$) \Comment{It takes $\Tinit(n,d)$ time}
    \State \blue{/*Initialize $\wt{S}_{r,\mathrm{fire}}$ and $S_{i,\mathrm{fire}}$ */}
    \State \Comment{It takes $\sum_{r=1}^{m}\Tquery(n,d, \wt{k}_{r,t}) = O(m \log n + m^{1/2}n)$ time} 
     \State  $\wt{S}_{r,\mathrm{fire}} \leftarrow \emptyset$ for $r \in [m]$. \Comment{$\wt{S}_{r,\mathrm{fire}}$ is the set of samples, for which neuron $r$ fires}\label{ln:init_b}
    \State $S_{i,\mathrm{fire}} \leftarrow \emptyset$ for $i \in [n]$. \Comment{$S_{i,\mathrm{fire}}$ is the set of neurons, which fire for $x_i$}
    \For{$r = 1 \to m$} 
        
        \State $\wt{S}_{r,\mathrm{fire}} \leftarrow \textsc{hsr}.\textsc{Query}(w_r(0),b)$ 
        \For{$i \in \wt{S}_{r,\mathrm{fire}}$}
            \State $S_{i, \mathrm{fire}}.\textsc{Add}(r)$
        \EndFor
    \EndFor\label{ln:init_e}
    
    \State \blue{/*Iterative step*/}
	\For{$t=1 \to T$}
	    \State \blue{/*Forward computation step*/}
	    \For{$i=1 \to n$}
		    \State $u(t)_i \leftarrow \frac{1}{ \sqrt{m} } \sum_{r\in \mathcal{S}_{i,\mathrm{fire}}} a_r \cdot \sigma_b(w_r(t)^\top x_i)$ \Comment{It takes $O(d\cdot k_{i,t})$ time}
		\EndFor
		\State \blue{/*Backward computation step*/}
		\State $P \leftarrow 0^{n \times m}$ \Comment{$P \in \R^{n \times m}$}
		\For{$i = 1 \to n$} \label{ln:classical_3}
		    \For{$r \in {\cal S}_{i,\mathrm{fire}}$}
			    \State $P_{i,r} \leftarrow \frac{1}{\sqrt{m}} a_r \cdot \sigma_b'( w_r(t)^\top x_i )$ 
			\EndFor
		\EndFor
		\State $M\gets X\diag(y-u(t))$ \Comment{$M\in \R^{d\times n}$, it takes $O(n\cdot d)$ time}
		\State $\Delta W \leftarrow \underbrace{M}_{d\times n}\underbrace{P}_{n\times m}$\label{ln:compute_delta_2}\Comment{$\Delta W\in \R^{d\times m}$, it takes $O(d\cdot \nnz(P))$ time, $\nnz(P) = O(nm^{4/5})$}
		\State $W(t+1)\gets W(t)-\eta \cdot \Delta W$.
		\State \blue{/*Update $\wt{S}_{r,\mathrm{fire}}$ and $S_{i,\mathrm{fire}}$ step*/}
		\State \Comment{It takes $O(\sum_{i=1}^n k_{i, t} + \sum_{r \in S_{[n], \mathrm{fire}}}\Tquery(n, d, \wt{k}_{r, t+ 1}))=O(n \cdot \log n \cdot m^{4/5})$ }
		\State $S_{[n], \mathrm{fire}} \leftarrow \cup_{i \in [n]}\mathcal{S}_{i, \mathrm{fire}}$
		\For{$r \in S_{[n], \mathrm{fire}}$} 
		\For{$i \in \wt{S}_{r,\mathrm{fire}}$}  \Comment{Removing old fired neuron indices. It takes $O(\wt{k}_{r, t})$ time} 
            \State $S_{i, \mathrm{fire}}.\textsc{Del}(r)$ 
        \EndFor
        \State $\wt{S}_{r,\mathrm{fire}} \leftarrow \textsc{hsr}.\textsc{Query}(w_r(t+1),b)$  \Comment{It takes $\Tquery(n,d, \wt{k}_{r,t+1})$ time}
        \For{$i \in \wt{S}_{r,\mathrm{fire}}$}  \Comment{Adding new fired neuron indices. It takes $O(\wt{k}_{r, t+1})$ time}
            \State $S_{i, \mathrm{fire}}.\textsc{Add}(r)$
        \EndFor
    \EndFor
	\EndFor
	\State \Return $W$ \Comment{$W \in \R^{d \times m}$}
	\EndProcedure
\end{algorithmic}
\end{algorithm}
\section{Preliminaries}\label{sec:probability_tools}

\paragraph{Notations}
For an integer $n$, we use $[n]$ to denote the set $\{1,2,\cdots,n\}$. For a vector $x$, we use $\| x \|_2$ to denote the entry-wise $\ell_2$ norm of a vector. We use $\E[]$ to denote the expectation and $\Pr[]$ to denote the probability. We use $M^\top$ to denote the transpose of $M$. We define matrix Frobenius norm as $\| M \|_F = ( \sum_{i,j} M_{i,j}^2 )^{1/2}$. We use $\| M \|$ to denote the operator norm of $M$. For $d\times m$ weight matrix $W$, we define $\|W \|_{\infty,2}:=\max_{r \in [m]} \| w_r \|_2$. We use $x^\top y$ to denote the inner product between vectors $x$ and $y$. We use $I_d$ to denote $d$-dimensional identity matrix. We use ${\cal N}(\mu ,\sigma^2)$ to denote Gaussian distribution with mean $\mu$ and variance $\sigma^2$.
 We use $\lambda_{\min}(M)$ and $\lambda_{\max}(M)$ to denote the minimum and the maximum eigenvalue of the matrix $M$, respectively. 

\subsection{Probabilities}

\begin{lemma}[Bernstein inequality \cite{b24}]\label{lem:bernstein}
Assume $Z_1, \cdots, Z_n$ are $n$ i.i.d. random variables. $\forall i \in [n]$, $\E[Z_i]=0$ and $|Z_i| \leq M$ almost surely. Let $Z = \sum_{i=1}^n Z_i$. Then,
\begin{align*}
\Pr \left[ Z > t \right] \leq \exp \left( - \frac{ t^2/2 }{ \sum_{j=1}^n \E[Z_j^2]  + M t /3 } \right), \forall t > 0.
\end{align*}
\end{lemma}

\begin{claim}[Theorem 3.1 in \cite{ls01}]\label{clm:gaussain_anti_shift}
Let $b>0$ and $r>0$. Then,
\begin{align*}
    \exp(-b^2/2)\Pr_{x\sim \N(0,1)}[|x|\leq r] \leq ~ \Pr_{x\sim \N(0,1)}[|x-b|\leq r] \leq ~ \Pr_{x\sim \N(0,1)}[|x|\leq r].
\end{align*}
\end{claim}

\begin{lemma}[Anti-concentration of Gaussian distribution]\label{lem:anti_gaussian}
Let $Z \sim {\N}(0,\sigma^2)$.
Then, for $t>0$,
\begin{align*}
    \Pr[|Z|\leq t]\leq \frac{2t}{\sqrt{2\pi}\sigma}.
\end{align*}
\end{lemma}

\begin{theorem}[Theorem 5.1.1 in \cite{tro15}]\label{thm:matrix_chernoff}
Let $X_1,\dots,X_m\in \R^{n\times n}$ be $m$ independent random Hermitian matrices. Assume that $0\preceq X_i \preceq L \cdot I$ for some $L>0$ and for all $i\in [m]$. Let $X:= \sum_{i=1}^m X_i$. Then, for $\epsilon \in (0, 1]$, we have
\begin{align*}
    \Pr[\lambda_{\min} (X) \leq \epsilon \lambda_{\min} (\E[X])]\leq n \cdot \exp(-(1-\epsilon)^2 \lambda_{\min} (\E[X]) / (2L)).
\end{align*}
\end{theorem}

\subsection{Half-space reporting data structures}
The time complexity of HSR data structure is:
\begin{theorem}[Agarwal, Eppstein and Matousek~\cite{aem92}]\label{thm:aem92}
Let $d$ be a fixed constant. Let $t$ be a parameter between $n$ and $n^{\lfloor d/2\rfloor}$. 
There is a dynamic data structure for half-space reporting that uses $O_{d,\epsilon}(t^{1+\epsilon})$ space and pre-processing time,  $O_{d,\epsilon}(\frac{n}{t^{1/\lfloor d/2\rfloor}}\log n+ k)$ time per query where $k$ is the output size and $\epsilon>0$ is any fixed constant, and $O_{d,\epsilon}(t^{1+\epsilon}/n)$ amortized update time. 
\end{theorem}
As a direct corollary, we have
\begin{corollary}[HSR data-structure time complexity \cite{aem92}]
Given a set of $n$ points in $\R^d$, the half-space reporting problem can be solved with the following performances:
\begin{itemize}
    \item Part 1.
    $\Tinit(n,d)=O_d(n\log n)$, $\Tquery(n,d,k)=O_{d,\epsilon}(n^{1-1/\lfloor d/2\rfloor+\epsilon}+k)$, amortized $\Tupdate=O_{d,\epsilon}(\log^2(n))$.
    \item Part 2.
    $\Tinit(n,d)=O_{d,\epsilon}(n^{\lfloor d/2\rfloor + \epsilon})$, $\Tquery(n,d,k)=O_{d,\epsilon}(\log(n)+k)$, amortized $\Tupdate=O_{d,\epsilon}(n^{\lfloor d/2\rfloor-1 + \epsilon})$.
\end{itemize}
\end{corollary}

\subsection{Basic algebras}

\begin{claim}[\cite{s11}]\label{clm:eigen_min}
Let $M_1,M_2\in \R^{n\times n}$ be two PSD matrices. Let $M_1 \circ M_2$ denote the Hadamard product of $M_1$ and $M_2$. Then,
\begin{align*}
    \lambda_{\min}(M_1 \circ M_2) \geq &~ (\min_{i\in [n]} {M_2}_{i,i})\cdot \lambda_{\min}(M_1), \\
    \lambda_{\max}(M_1 \circ M_2) \leq &~ (\max_{i\in [n]} {M_2}_{i,i})\cdot \lambda_{\max}(M_1).
\end{align*}
\end{claim}

\section{Sparsity Analysis}\label{sec:missing_proof}

\subsection{Bounding difference between continuous kernel and discrete kernel}

In \cite{dzps19,sy19}, they proved the following lemma for $b=0$. Here, we provide a more general statement for any $b \geq 0$.
\begin{lemma}\label{lem:3.1}
For any shift parameter $b\geq 0$, we define continuous version of shifted NTK $H^{\cts}$ and discrete version of shifted NTK $H^{\dis}$ as:
\begin{align*}
H^{\cts}_{i,j} := & ~ \E_{w \sim \N(0,I)} \left[ x_i^\top x_j {\bf 1}_{ w^\top x_i \geq b, w^\top x_j \geq b } \right] , \\ 
H^{\dis}_{i,j} := & ~ \frac{1}{m} \sum_{r=1}^m \left[ x_i^\top x_j {\bf 1}_{ w_r^\top x_i \geq b, w_r^\top x_j \geq b } \right].
\end{align*}
We define $\lambda := \lambda_{\min} (H^{\cts}) $. 

Let $m = \Omega( \lambda^{-1} n\log (n/\rho) )$ be number of samples of $H^{\dis}$, then
\begin{align*}
\Pr \Big[ \lambda_{\min} ( H^{\dis} ) \geq \frac{3}{4} \lambda \Big] \geq 1-\rho.
\end{align*}
\end{lemma}

\begin{proof}
We will use the matrix Chernoff bound (Theorem~\ref{thm:matrix_chernoff}) to provide a lower bound on the least eigenvalue of discrete version of shifted NTK $H^{\dis}$.

Let $H_r := \frac{1}{m} \wt{X}(w_r) \wt{X}(w_r)^\top$, where $\wt{X}(w_r)\in \R^{d\times n}$ is defined as follows: 
\begin{align*}
    \wt{X}(w_r) = \begin{bmatrix}
    \mathbf{1}_{w_r^\top x_i\geq b}\cdot x_1 & \cdots &  \mathbf{1}_{w_r^\top x_n\geq b} \cdot x_n
    \end{bmatrix}.
\end{align*}
Hence, $H_r \succeq 0$. We need to upper-bound $\|H_r\|$. Naively, we have
\begin{align*}
    \|H_r\| \leq \|H_r\|_F \leq \frac{n}{m},
\end{align*}
since for each entry at $(i,j)\in [n]\times [n]$, 
\begin{align*}
    (H_r)_{i,j}=\frac{1}{m}x_i^\top x_j \mathbf{1}_{w_r^\top x_i \geq b, w_r^\top x_j\geq b}\leq \frac{1}{m}x_i^\top x_j\leq \frac{1}{m}.
\end{align*}

Then, $H^{\dis} = \sum_{r=1}^m H_r$, and $\E[H^{\dis}]=H^{\cts}$. And we assume that $\lambda_{\min}(H^{\cts}) = \lambda$. 

Hence, by matrix Chernoff bound (Theorem~\ref{thm:matrix_chernoff}) and choosing choose $m = \Omega(\lambda^{-1}
n \cdot \log(n/\rho))$, we can show
\begin{align*}
   \Pr\left[\lambda_{\min} (H^{\dis}) \leq \frac{3}{4} \lambda\right]\leq &~ n \cdot \exp(-\frac{1}{16} \lambda / (2n/m))\\
   = &~ n \cdot \exp(-\frac{\lambda m}{32n})\\
   \leq &~ \rho,
\end{align*}
Thus, we finish the proof.
\end{proof}

\subsection{Handling Hessian if perturbing weight}

We present a tool which is inspired by a list of previous work \cite{dzps19,sy19}.

\begin{lemma}[perturbed $w$ for shifted NTK]\label{lem:3.2}
Let $b>0$ and $R\leq 1/b$. Let $c>0$ and $c'>0$ denote two fxied constants. We define function $H$ that is mapping $\R^{m \times d}$ to $\R^{n \times n}$ as follows:
\begin{align*}
   \text{~the~$(i,j)$-th~entry~of~} H(W) \text{~is~} \frac{1}{m} x_i^\top x_j \sum_{r=1}^m {\bf 1}_{ w_r^\top x_i \geq b, w_r^\top x_j \geq b } .
\end{align*} 
Let $\wt{W} \in \R^{d \times m}$ be $m$ vectors that are sampled from ${\N}(0,I_d)$. Consider   $W \in \R^{d \times m}$ that satisfy, $\| \wt{W} - W \|_{\infty,2} \leq R$, it has
\begin{itemize}
    \item Part 1, $\|  H(\wt{W}) - H (W) \|_F \leq n \cdot \min\{ c\cdot \exp(-b^2/2), 3R\}$ holds with probability at least $1-n^2\cdot \exp(-m\cdot \min\{c'\cdot  \exp(-b^2/2), R/10\})$.
    \item Part 2, $\lambda_{\min}(H(W))\geq \frac{3}{4}\lambda - n \cdot \min\{ c\cdot \exp(-b^2/2), 3R\}$ holds with probability at least $ 1-n^2\cdot \exp(-m\cdot \min\{c'\cdot  \exp(-b^2/2), R/10\})-\rho$.
\end{itemize}
\end{lemma}
\begin{proof}

Consider
{
\begin{align*}
\| H(W) - H(\wt{W})  \|_F^2 = & ~ \sum_{i \in [n]} \sum_{j \in [n]} ( H(\wt{W})_{i,j} - H(W)_{i,j} )^2 \\
\leq & ~ \frac{1}{m^2} \sum_{i \in [n]} \sum_{j \in [n]} \left( \sum_{r\in [m]} {\bf 1}_{ \wt{w}_r^\top x_i \geq b, \wt{w}_r^\top x_j \geq b} - {\bf 1}_{ w_r^\top x_i \geq b , w_r^\top x_j \geq b } \right)^2 \\
= & ~ \frac{1}{m^2} \sum_{i\in [n]} \sum_{j\in[n]}  \Big( \sum_{r\in[m]} s_{r,i,j,b} \Big)^2 ,
\end{align*}
}

where the first step follows from definition of Frobenius norm, the last third step follows from by defining
\begin{align*}
s_{r,i,j,b} :=  {\bf 1}_{ \wt{w}_r^\top x_i \geq b, \wt{w}_r^\top x_j \geq b} - {\bf 1}_{ w_r^\top x_i \geq b , w_r^\top x_j \geq b } .
\end{align*} 

 For simplicity, we use $s_r$ to $s_{r,i,j,b}$ (note that we fixed $(i,j)$ and $b$).


Define $A_{i,r}$ to be the event that 
\begin{align*}
    A_{i,r}=\{\exists w\in \R^d: \|w-w_r\|_2 \leq R, \mathbf{1}_{w^\top x_i\geq b}\ne \mathbf{1}_{w_r^\top x_i \geq b}\}.
\end{align*}
Note that event $A_{i,r}$ happens iff $|w_r^\top x_i - b|\leq R$ happens.

Prior work \cite{dzps19,sy19} only one way to bound $\Pr[A_{i,r}]$. We present two ways of arguing the upper bound on $\Pr[ A_{i,r} ]$. One is anti-concentration, and the other is concentration.

By anticoncentration, (Lemma~\ref{lem:anti_gaussian}),
\begin{align*}
    \Pr[A_{i,r}]\leq \frac{2R}{\sqrt{2\pi}} \leq R.
\end{align*}
By concentration,
\begin{align*}
    \Pr[A_{i,r}]\leq \exp(-(b-R)^2/2) \leq c_1\cdot \exp(-b^2/2).
\end{align*}
where the last step follows from  $R<1/b$ and $c_1\geq \exp(1-R^2/2)$ is a constant.

Hence,
\begin{align*}
    \Pr[A_{i,r}] \leq \min \{ R, c_1\exp( - b^2 / 2 ) \}.
\end{align*}
If the event $\neg A_{i,r}$ happens and the event $\neg A_{j,r}$ happens, then we have
\begin{align*}
\left| {\bf 1}_{ \wt{w}_r^\top x_i \geq b, \wt{w}_r^\top x_j \geq b} - {\bf 1}_{ w_r^\top x_i \geq b, w_r^\top x_j \geq b } \right|=0.
\end{align*}
If the event $A_{i,r}$ happens or the event  $A_{j,r}$ happens, then we obtain
\begin{align*}
\left| {\bf 1}_{ \wt{w}_r^\top x_i \geq b, \wt{w}_r^\top x_j \geq b} - {\bf 1}_{ w_r^\top x_i \geq b , w_r^\top x_j \geq b } \right|\leq 1.
\end{align*}
\paragraph{Case 1: $c_1\exp( - b^2 / 2 ) < R$.}

So we have 
\begin{align*}
 \E_{\wt{w}_r}[s_r] 
 \leq & ~ \Pr[A_{i,r}]+\Pr[A_{j,r}] \\
 \leq & ~ c_1 \cdot \exp( - b^2 / 2 ) 
\end{align*}

Now, we calculate the variance
\begin{align*}
    \E_{\wt{w}_r} \left[ \left(s_r-\E_{\wt{w}_r}[s_r] \right)^2 \right]
    = & ~ \E_{\wt{w}_r}[s_r^2]-\E_{\wt{w}_r}[s_r]^2 \\
    \leq & ~ \E_{\wt{w}_r}[s_r^2]\\
    \leq & ~\E_{\wt{w}_r} \left[ \left( {\bf 1}_{A_{i,r}\vee A_{j,r}} \right)^2 \right] \\
    \leq &~ c_1\cdot \exp(-b^2/2).
\end{align*}
Note that $|s_r|\leq 1$ for all $r$.

Define $\ov{s} = \frac{1}{m}\sum_{r=1}^m s_r$. Thus, we are able to use Lemma~\ref{lem:bernstein},
\begin{align*}
    \Pr \left[ m \cdot \ov{s} \geq m\cdot c_1\exp(-b^2/2) +mt \right]
    \leq & ~ \Pr \left[\sum_{r=1}^m (s_r-\E[s_r])\geq mt \right]\\
    \leq & ~ \exp \left( - \frac{ m^2t^2/2 }{ m\cdot c_1\exp(-b^2/2)   + mt/3 } \right), ~~~ \forall t \geq 0.
\end{align*}

Define $\ov{s} = \frac{1}{m} \sum_{r=1}^m s_r$. Thus, it gives

\begin{align*}
\Pr \left[ \ov{s} \geq c_2\cdot  \exp(-b^2/2) \right] \leq \exp( - c_3\cdot m\exp(-b^2/2)),
\end{align*}
where $c_2:=2c_1,c_3:=\frac{3}{8}c_1$ are some constants.

\paragraph{Case 2: $\exp(-b^2/2) > R$.}
Then, we have
\begin{align*}
    \E_{\wt{w}_r}[s_r]\leq 2R, ~~~\E_{\wt{w}_r} \left[ \left(s_r-\E_{\wt{w}_r}[s_r] \right)^2 \right] \leq 2R.
\end{align*}
Define $\ov{s} = \frac{1}{m}\sum_{r=1}^m s_r$. By Lemma~\ref{lem:bernstein}, 
\begin{align*}
    \Pr \left[ \ov{s} \geq 3R  \right]
    \leq & ~ \exp \left( -mR/10 \right).
\end{align*}
{\bf Combining two cases:}


Thus, we obtain 
\begin{align*}
  & ~ \Pr \Big[ \|H(\wt{W}) - H (W)   \|_F \leq n \cdot \min\{ c_2\exp(-b^2/2), 3R\} \Big] \\
   \geq & ~1-n^2\cdot \exp(-m\cdot \min\{c_3 \exp(-b^2/2), R/10\}).
\end{align*}

For the second part, by Lemma~\ref{lem:3.2},  $\Pr[ \lambda_{\min}(H(\wt{W}))\geq 0.75 \cdot \lambda ] \geq 1- \rho$. Hence,
\begin{align*}
    \lambda_{\min}(H(W))\geq&~  \lambda_{\min}(H(\wt{W})) - \|H(\wt{W}) - H(W)\|\\
    \geq &~ \lambda_{\min}(H(\wt{W})) - \|H(\wt{W}) - H(W)\|_F\\ \geq &~ 0.75 \cdot \lambda - n \cdot \min\{ c_2\cdot \exp(-b^2/2), 3R\},
\end{align*}
which happens with probability $1-n^2\cdot \exp(-m\cdot \min\{c_3\cdot  \exp(-b^2/2), R/10\}) - \rho$ by the union bound.
\end{proof}

\subsection{Total movement of weights}

\begin{definition}[Hessian matrix at time $t$]
For $t\geq 0$, let $H(t)$ be an $n\times n$ matrix with $(i,j)$-th entry:
\begin{align*}
    H(t)_{i,j}:= \frac{1}{m}x_i^\top x_j\sum_{r=1}^m \mathbf{1}_{ \langle w_r(t), x_i \rangle \geq b } \mathbf{1}_{ \langle w_r(t), x_j \rangle \geq b } 
\end{align*}
\end{definition}

We follow the standard notation $D_{\cts}$ in Lemma 3.5 in \cite{sy19}.
\begin{definition}[$D_{\cts}$]
Let $y\in \R^n$ be the vector of the training data labels. Let $\err(0)\in \R^n$ denote the error of prediction of the neural network function (Definition~\ref{def:err}). Define the actual moving distance of weight $D_{\cts}$ to be
\begin{align*}
D_{\cts} := \lambda^{-1} \cdot m^{-1/2} \cdot  \sqrt{n} \cdot \| \err(0) \|_2 .
\end{align*}
\end{definition}

We state a tool from previous work \cite{dzps19,sy19} (more specifically, Lemma 3.4 in \cite{dzps19}, Lemma 3.6 in \cite{sy19}). Since adding the shift parameter $b$ to NTK doesn't affect the proof of the following lemma, thus we don't provide a proof and refer the readers to prior work.
\begin{lemma}[\cite{dzps19,sy19}]\label{lem:3.4}
The condition $D_{\cts}<R$ implies
$\lambda_{\min}(H(t))\geq \lambda/2$, $\forall t \geq 0$. Let $\err(t)$ be defined as Definition~\ref{def:err}. Further,
\begin{enumerate}
    \item $\|W(t)-W(0)\|_{\infty,2} \leq  D_{\cts}$,
    \item $\|\err(t)\|_2^2 \leq  \exp(-\lambda t) \cdot \|\err(0)\|_2^2$.
\end{enumerate}
\end{lemma}

\subsection{Bounded gradient}
The proof of Lemma 3.6 in \cite{sy19} implicitly implies the following basic property of gradient.
\begin{claim}[Bounded gradient]\label{cla:gradient_bound}
Let $\err(s)$ be defined as Definition~\ref{def:err}. 
For any $0 \leq s \leq t$, We have
\begin{align*}
\left\|\frac{\partial L(W(s))}{ \partial w_r(s) } \right\|_2  \leq & ~\frac{\sqrt{n}}{ \sqrt{m} } \| \err(s) \|_2  \\
   \left\| \frac{\d}{ \d s} w_r(s) \right\|_2 \leq & ~ \frac{\sqrt{n}}{ \sqrt{m} } \| \err(s) \|_2 
\end{align*}

\end{claim}
\begin{proof}

For the first part,
\begin{align*}
 \Big\|\frac{\partial L(W(s))}{ \partial w_r(s) } \Big\|_2 \notag 
= & ~ \Big\| \sum_{i=1}^n \err_i(s) \frac{1}{\sqrt{m}} a_r x_i \cdot {\bf 1}_{ w_r(s)^\top x_i \geq b } \Big\|_2 & \text{~by~Eq.~\eqref{eq:wr_derivative}} \notag\\
\leq & ~ \frac{1}{ \sqrt{m} } \sum_{i=1}^n | \err_i(s) | & \text{~by~Eq.~\eqref{eq:gradient}} \notag\\
\leq & ~ \frac{ \sqrt{n} }{ \sqrt{m} } \| \err(s) \|_2 .
\end{align*}

For second part, we use ODE to prove it.

\end{proof}

\subsection{Upper bound on the movement of weights per iteration}

The following Claim is quite standard in the literature, we omitt the details.
\begin{claim}[Corollary 4.1 in \cite{dzps19}, Lemma 3.8 in \cite{sy19}]\label{clm:4.1}
Let $\err(i)$ be defined as Definition~\ref{def:err}. 
If $\forall i \in[t], \| \err(i) \|_2^2 \leq ( 1 - \eta \lambda / 2 )^i \cdot \| \err(0) \|_2^2$, then 
\begin{align*}
 \| W(t+1) - W_r(0) \|_{\infty,2} \leq  4 \lambda^{-1} m^{-1/2} \cdot \sqrt{n} \cdot \| \err(0) \|_2  := D.
\end{align*}
\end{claim}

\subsection{Bounding the number of fired neuron per iteration}

In this section, we will show that for $t=0,1,\dots,T$, the number of fire neurons $k_{i,t}=|{\cal S}_{i,\mathrm{fire}}(t)|$ is small with high probability.

We define the set of neurons that are flipping at time $t$:
\begin{definition}[flip set]
For each $i \in [n]$, for each time $t\in [T]$ let ${\cal S}_{i,\mathrm{flip}}(t) \subset [m]$ denote the set of neurons that are never flipped during the entire training process,
\begin{align*}
    {\cal S}_{i,\mathrm{flip}}(t) := & ~ \{ r \in [m] :  
     ~ \sgn(\langle w_r(t), x_i \rangle - b ) \neq \sgn( \langle w_r(t-1), x_i \rangle - b )  \} . 
\end{align*}
\end{definition}

Over all the iterations of training algorithm, there are some neurons that never flip states. We provide a mathematical formulation of that set, 
\begin{definition}[noflip set]
For each $i \in [n]$, let $S_i \subset [m]$ denote the set of neurons that are never flipped during the entire training process,
\begin{align}\label{eq:noflip_def}
    S_i := & ~ \{ r \in [m] :  \forall t \in [T] 
     ~ \sgn(\langle w_r(t), x_i \rangle - b) = \sgn( \langle w_r(0), x_i \rangle - b ) \} . 
\end{align}
\end{definition}

In Lemma~\ref{lem:sparse_initial}, 
we already show that $k_{i,0}=O(m\cdot \exp(-b^2/2))$ for all $i\in [n]$ with high probability. We can show that it also holds for $t>0$.

\begin{lemma}[Bounding the number of fired neuron per iteration]\label{lem:bound_fire_neurons}
Let $b\geq 0$ be a parameter, and let $\sigma_b(x)=\max\{x, b\}$ be the activation function. For each $i \in [n], t\in[T]$, $k_{i,t}$ is the number of activated neurons at the $t$-th iteration.  For $0<t\leq T$, with probability at least $1-n \cdot \exp \left(-\Omega(m)\cdot \min\{R, \exp(-b^2/2)\} \right)$, 
$k_{i,t}$ is at most $O(m\exp(-b^2/2))$ for all $i\in [n]$.
\end{lemma}

\begin{proof}
We prove this lemma by induction. 

The base case of $t=0$ is shown by Lemma~\ref{lem:sparse_initial} that $k_{i,0}=O(m\cdot \exp(-b^2/2))$ for all $i\in [n]$ with probability at least $1-n\exp(-\Omega(m\cdot \exp(-b^2/2)))$.

Assume that the statement holds for $0,\dots,t-1$. By Claim~\ref{clm:4.1}, we know  $\forall k<t$,
\begin{align*}
     \| W(k+1) - W(0) \|_{\infty,2} < R.
\end{align*}

Consider the $t$-th iteration. For each $i\in [n]$, consider the set of activated neurons ${\cal S}_{i, \mathrm{fire}}$. We note that for the neurons in $S_i$, with high probability these neurons will not be activated in the $t$-th iteration if they are not activated in the $(t-1)$-th iteration. By Claim~\ref{clm:S_condition}, for $r\in [m]$, 
\begin{align*}
    \Pr[r\notin S_i]\leq \min\left\{R, O(\exp(-b^2/2))\right\}.
\end{align*}

On the one hand, if $R < O(\exp(-b^2/2))$, then $\E[|\ov{S_i}|] \leq mR$. By  Lemma~\ref{lem:bernstein}, 
\begin{align*}
  \Pr \big[ |\ov{S}_i| > t \big] \leq &~ \exp \left(-\frac{t^2/2}{mR + t/3} \right).
\end{align*}
If we take $t:=mR$, then we have
\begin{align*}
    \Pr \big[ |\ov{S}_i| > mR \big] \leq &~ \exp \left(-3mR/8 \right).
\end{align*}

On the other hand, if $O(\exp(-b^2/2)) < R$, then $\E[|\ov{S_i}|] \leq O(m\exp(-b^2/2))$.
By Lemma~\ref{lem:bernstein}, we have that
\begin{align*}
  \Pr \big[ |\ov{S}_i| > t \big] \leq &~ \exp\left(-\frac{t^2/2}{O(m\exp(-b^2/2)) + t/3} \right).
\end{align*}
If we take $t:=m\exp(-b^2/2)$, we have that
\begin{align*}
    \Pr \big[ |\ov{S}_i| > m\exp(-b^2/2) \big] \leq &~ \exp(-\Omega(m\exp(-b^2/2))).
\end{align*}

Then, we know that in addition to the fire neurons in $S_{i,\mathrm{noflip}}$, there are at most $m\cdot \min\{R, \exp(-b^2/2)\}$ neurons are activated in $t$-th iteration with high probability. 

By a union bound for $i\in [n]$, we obtain with probability 
\begin{align*}
\geq 1-n \cdot \exp( -\Omega(m)\cdot \min\{R, \exp(-b^2/2)\} ),
\end{align*}
the number of activated neurons for $x_i$ at the $t$-th iteration of the algorithm is
\begin{align*}
    k_{i,t} = |S_{i,\mathrm{fire}}(t)| \leq k_{i,0} + m\min\{R, \exp(-b^2/2)\} \leq O(m\exp(-b^2/2)),
\end{align*}
where the last step follows from $k_{i,0}=O(m\exp(-b^2/2))$ by Lemma~\ref{lem:sparse_initial}.

The Lemma is then proved for all $t=0,\dots,T$.
\end{proof}

\begin{claim}[Bound on noflip probability]\label{clm:S_condition}
Let $R \leq 1/b$. 
For $i\in [n]$, let $S_i$ be the set defined by Eq.~\eqref{eq:noflip_def}.  \\
{\bf Part 1.} For $r\in [m]$, $r\notin S_i$ if and only if $| \langle w_r(0), x_i \rangle -b| < R$. \\
{\bf Part 2.} If $w_r(0)\sim \mathcal{N}(0, I_d)$, then
\begin{align*}
    \Pr[r\notin S_i]\leq \min \{ R, O(\exp(-b^2/2)) \} ~~~\forall r\in [m].
\end{align*}
\end{claim}
\begin{proof}

{\bf Part 1.}
We first note that $r\notin S_i \subset [m]$ is equivalent to the event that 
\begin{align*}
   \exists w \in \R^d, \mathrm{s.t.}  \mathbf{1}_{ \langle w_r(0) , x_i \rangle \geq b}\ne \mathbf{1}_{ \langle w, x_i \rangle \geq b} \land \| w - w_r(0) \|_2 < R.
\end{align*}
 Assume that $\|w-w_r(0)\|_2 = R$. Then, we can write $w = w_r(0) + R \cdot v$ with $\|v\|_2 = 1$ and $\langle w , x_i \rangle = \langle w_r(0), x_i \rangle + R \cdot \langle v, x_i \rangle$.

Now, suppose there exists a $w$ such that $\mathbf{1}_{ \langle w_r(0) , x_i \rangle \geq b}\ne \mathbf{1}_{ \langle w, x_i \rangle \geq b}$.
\begin{itemize}
\item If $\langle w_r(0), x_i \rangle > b$, then there exists a vector $v \in \R^d$ such that $R \cdot \langle v , x_i \rangle <  b - \langle w_r(0), x_i \rangle$,
\item If $\langle w_r(0), x_i \rangle < b$, then there exists a vector $v \in \R^d$ such that $R \cdot \langle v, x_i \rangle > b- \langle w_r(0), x_i \rangle$.
\end{itemize}

Since $\|x_i\|_2=1$ and $\langle v, x_i\rangle \in [-1, 1]$, we can see that the above conditions hold if and only if  
\begin{align*}
    & b - \langle w_r(0), x_i \rangle > -R, \quad \text{and}\\
    & b - \langle w_r(0), x_i \rangle < + R.
\end{align*}
In other words, $r\notin S_i$ if and only if $| \langle w_r(0), x_i \rangle - b| < R$.

{\bf Part 2.}
 
We have
\begin{align*}
\Pr[ r\notin S_i]
= & ~ \Pr_{ z \sim \N(0,1) } [ | z -b| < R ] & \text{~by~$\langle w_r, x_i \rangle \sim {\cal N}(0, 1)$} \\
\leq & ~ \Pr_{ z \sim \N(0,1) } [ | z| < R ] & \text{~by~symmetric~property~of~Gaussian~distribution} \\
\leq & ~ \frac{ 2 R }{ \sqrt{2\pi} }& ~ \text{~by~anti-concentration~inequality~of~Gaussian (Lemma~\ref{lem:anti_gaussian})} \\
\leq & ~ R.
\end{align*}

On the other hand, we also know
\begin{align*}
    \Pr[r\notin S_i] \leq \Pr_{z\sim {\cal N}(0, 1)}[z\geq b-R]\leq \exp(-(b-R)^2/2)\leq O(\exp(-b^2/2)),
\end{align*}
where the last step follows from $R<1/b$.
\end{proof}

\newpage
\section{Convergence Analysis}\label{sec:converge_analysis}
\subsection{Upper bound the initialization}\label{sec:proof_yu0}

The following Claim provides an upper bound for initialization. Prior work only shows it for $b=0$, we generalize it to $b \geq 0$. The modification to the proof of previous Claim 3.10 in \cite{sy19} is quite straightforward, thus we omit the details here.
 \begin{claim}[Upper bound the initialization, shited NTK version of Claim 3.10 in \cite{sy19}]\label{cla:yu0}
 Let $b \geq 0$ denote the NTK shifted parameter. 
Let parameter $\rho \in (0,1)$ denote the failure probability. Then
\begin{align*}
\Pr[ \|\err(0)\|_2^2=O(n(1+b^2)\log^2(n/\rho)) ] \geq 1- \rho.
\end{align*}
\end{claim}

\subsection{Bounding progress per iteration}
In previous work, \cite{sy19} define $H$ and $H^{\bot}$ only for $b=0$. In this section, we generalize it to $b \geq 0$.
Let us define two shifted matrices $H$ and $H^{\bot} $
\begin{align}
H(k)_{i,j} := & ~ \frac{1}{m} \sum_{r=1}^m \langle x_i , x_j \rangle {\bf 1}_{ \langle w_r(k) , x_i \rangle \geq b, \langle w_r(k), x_j \rangle \geq b } , \label{eq:def_H}\\
H(k)^{\bot}_{i,j} := & ~ \frac{1}{m} \sum_{r\in \ov{S}_i} \langle x_i , x_j \rangle {\bf 1}_{ \langle w_r(k), x_i \rangle \geq b, \langle w_r(k) , x_j \rangle \geq b} .\label{eq:def_H_bot}
\end{align}

We define
\begin{align}
    v_{1,i} := & ~ \frac{1}{\sqrt{m}} \sum_{r \in S_i} a_r ( \sigma_b( w_{r}(k+1)^\top x_i ) - \sigma_b(w_r(k)^\top x_i ) ) \notag\\
    v_{2,i} := & ~ \frac{1}{\sqrt{m}} \sum_{r \in \ov{S}_i} a_r ( \sigma_b( w_{r}(k+1)^\top x_i ) - \sigma_b(w_r(k)^\top x_i ) ) \label{eq:def_v_2_i}
\end{align}
Following the same proof as Claim~3.9 \cite{sy19}, we can show that the following Claim. The major difference between our claim and Claim 3.9 in \cite{sy19} is, they only proved it for the case $b = 0$. We generalize it to $b \geq 0$. The proof is several basic algebra computations, we omit the details here.
\begin{claim}[Shifted NTK version of Claim~3.9 in \cite{sy19}]\label{cla:split_into_four_terms}
Let $\err(k) =  y -u(k) $ be defined as Definition~\ref{def:err}.
\begin{align*}
\| \err(k+1) \|_2^2 = \| \err(k) \|_2^2 + B_1 + B_2 + B_3 + B_4,
\end{align*}
where
\begin{align*}
B_1 := & ~ -2 \eta \cdot \err(k)^\top \cdot H(k) \cdot \err(k)  , \\
B_2 := & ~ + 2 \eta \cdot \err(k) ^\top \cdot H(k)^{\bot} \cdot \err(k)  , \\
B_3 := & ~ - 2  \err(k) ^\top v_2 , \\
B_4 := & ~ + \| u (k+1) - u(k) \|_2^2 . 
\end{align*}
\end{claim}

The nontrivial parts in our analysis is how to bound $B_1, B_2, B_3$ and $B_4$ for the shifted cases (We will provide a proof later). Once we can bound all these terms, we can show the following result for one iteration of the algorithm:
\begin{lemma}[Shifted NTK version of Page 13 in \cite{sy19}]\label{cla:inductive_claim}

We have
\begin{align*}
    \| \err(k+1) \|_2^2 \leq \| \err(k) \|_2^2 \cdot ( 1 - \eta \lambda + 4\eta n\cdot \min\{R, \exp(-b^2/2)\}  + \eta^2 n^2 ))
\end{align*}
holds with probability at least
\begin{align*}
    1-2n^2\cdot \exp(-\Omega(m)\cdot \min\{R,  \exp(-b^2/2)\})-\rho.
\end{align*}
\end{lemma}

\begin{proof}

We are able to provide the following upper bound for $\| \err(k+1) \|_2^2$: 
\begin{align*}
&~\| \err(k+1) \|_2^2 \\
= & ~ \| \err(k) \|_2^2 + B_1 + B_2 + B_3 + B_4 & ~\text{~by~Claim~\ref{cla:split_into_four_terms}} \\
\leq & ~ \| \err(k) \|_2^2 ( 1 - \eta \lambda + 4\eta n\cdot \min\{R, \exp(-b^2/2)\}  + \eta^2 n^2 ) & ~ \text{~by~Claim~\ref{cla:C1}, \ref{cla:C2}, \ref{cla:C3} and \ref{cla:C4}} 
\end{align*}
\end{proof}

\subsection{Upper bound on the norm of dual Hessian
}

The proof of the following fact is similar to Fact C.1 in \cite{sy19}. We generalize the $b=0$ to $b \geq 0$. The same bound will hold as Fact C.1 in \cite{sy19} if we replace ${\bf 1}_{ w_r(k)^\top x_i \geq 0 }$ by ${\bf 1}_{ w_r(k)^\top x_i \geq b }$. Thus, we omit the details here.
\begin{fact}[Shifted NTK version of Fact C.1 in \cite{sy19}]\label{fact:bound_H_k_bot}
Let $b \geq 0$. Let shifted matrix $H(k)^{\bot}$ be defined as Eq.~\eqref{eq:def_H_bot}. For all $k\geq 0$, we have
\begin{align*}
\| H(k)^{\bot} \|_F \leq \frac{n}{m^2} \sum_{i=1}^n |\ov{S}_i|^2 .
\end{align*}
\end{fact}

\subsection{Bounding the gradient improvement term 
}\label{sec:proof_c1}

\begin{claim}[Bounding the gradient improvement term]\label{cla:C1}
Let $H(k)$ be shifted matrix  (see Eq.~\eqref{eq:def_H}). Assume $b \geq 0$.  Denote $\rho_0=n^2\cdot \exp(-m\cdot \min\{c'\cdot  \exp(-b^2/2), R/10\})+\rho$.
We define $B_1 := -2 \eta \err(k)^\top H(k) \err(k)$. Assuming either of the following condition,
\begin{itemize}
\item $R\leq \frac{\lambda}{12n}$,
\item  $b\geq \sqrt{2\cdot \log (4cn/\lambda)}$.
\end{itemize}
Then, we have
\begin{align*}
\Pr[ B_1 \leq - \eta \lambda\cdot \| \err(k) \|_2^2 ] \geq 1- \rho_0. 
\end{align*}
\end{claim}

\begin{proof}
By Lemma \ref{lem:3.2}, there exists constants $c,c'>0$ such that
\begin{align*}
    \lambda_{\min}(H(W))\geq \frac{3}{4}\lambda - n \cdot \min\{ c\cdot \exp(-b^2/2), 3R\}
\end{align*}
with probability at least $1-\rho_0$. 

If we have $R\leq \frac{\lambda}{12n}$ or $b\geq \sqrt{2\cdot \log (4cn/\lambda)}$, then
\begin{align*}
    \lambda_{\min}(H(W)) \geq \frac{1}{2}\lambda.
\end{align*}

Finally, we have
\begin{align*}
  \err(k)^\top \cdot  H(k) \cdot \err(k) \geq \|\err(k) \|_2^2 \cdot \lambda / 2.
\end{align*}
\end{proof}

\subsection{Bounding the blowup by the dual Hessian term 
}\label{sec:proof_c2}

\begin{claim}[Bounding the blowup by the dual Hessian term]\label{cla:C2}
Let shifted matrix $H(k)^{\bot}$ be defined as Eq.~\eqref{eq:def_H_bot}. Let $\rho_0 = n\exp(-\Omega(m)\cdot \min\{R, \exp(-b^2/2)\})$. Let $b \geq 0$ be shifted NTK parameter. 
We define $B_2 := 2 \eta \cdot \err(k)^\top \cdot H(k)^{\bot} \cdot \err(k)$. Then
\begin{align*}
\Pr[ B_2 \leq 2\eta n\cdot \min\{R, \exp(-b^2/2)\}\cdot \| \err(k) \|_2^2 ] \geq 1-\rho_0.
\end{align*}
\end{claim}

\begin{proof}
By property of spectral norm,
\begin{align*}
B_2 \leq 2 \eta \| \err(k) \|_2^2 \| H(k)^{\bot} \|.
\end{align*}


Using Fact~\ref{fact:bound_H_k_bot}, we have $\| H(k)^{\bot} \|_F \leq \frac{n}{m^2} \sum_{i=1}^n |\ov{S}_i|^2 $.

By Lemma~\ref{lem:bound_fire_neurons}, $\forall i\in \{1,2,\cdots,n\}$, it has 
\begin{align}\label{eq:bound_ov_S}
    \Pr\left[|\ov{S_i}|\leq m\cdot \min\{R, \exp(-b^2/2)\}\right]\geq 1- \rho_0. 
\end{align}
Hence,
with probability at least $1-\rho_0$ 
\begin{align*}
\| H(k)^{\bot} \|_F^2\leq \frac{n}{m^2}\cdot n\cdot m^2\cdot \min\{R^2, \exp(-b^2)\}=n^2\cdot \min\{R^2, \exp(-b^2)\} .
\end{align*}
Putting all together, we have
\begin{align*}
\| H(k)^{\bot} \|\leq \| H(k)^{\bot} \|_F\leq n\cdot \min\{R, \exp(-b^2/2)\}
\end{align*}
with probability at least $1-\rho_0$.  

\end{proof}

\subsection{Bounding the blowup by the flip-neurons term
}\label{sec:proof_c3}

\begin{claim}[Bounding the blowup by flipping neurons term]\label{cla:C3}
Let $\rho_0 = n\exp (-\Omega(m)\cdot \min\{R, \exp(-b^2/2)\}  )$. 
We define $B_3 := - 2 \err(k)^\top v_2$. Let $b \geq 0$ be shifted NTK parameter. Then we have
\begin{align*}
\Pr[ B_3 \leq 2\eta n \cdot \min\{R, \exp(-b^2/2)\}\cdot \|\err(k)\|_2^2 ] \geq 1- \rho_0.
\end{align*}
\end{claim}

\begin{proof}
Using Cauchy-Schwarz inequality, we have
$
B_3 \leq 2 \| \err(k) \|_2 \cdot \| v_2 \|_2
$. 

Then we focus on $\|v_2\|_2$,
\begin{align*}
\| v_2 \|_2^2
\leq &~ \sum_{i=1}^n \left(\frac{\eta}{ \sqrt{m} } \sum_{ r \in \ov{S}_i } \left| ( \frac{ \partial L(W(k)) }{ \partial w_r(k) } )^\top x_i \right|\right)^2 & \text{~by~Eq.~\eqref{eq:def_v_2_i}} \\
= &~ \frac{\eta^2}{ m }\sum_{i=1}^n \left(\sum_{r=1}^m \mathbf{1}_{r\in \ov{S}_i}\left| ( \frac{ \partial L(W(k)) }{ \partial w_r(k) } )^\top x_i \right|\right)^2\\
\leq &~ \frac{\eta^2}{ m }\cdot \max_{r \in [m]} \left|  \frac{ \partial L(W(k)) }{ \partial w_r(k) } \right|^2\cdot\sum_{i=1}^n \left(\sum_{r=1}^m \mathbf{1}_{r\in \ov{S}_i}\right)^2\\
 \leq & ~ \frac{\eta^2}{ m }\cdot (\frac{ \sqrt{n} }{ \sqrt{m} } \| \err(k) \|_2 )^2 \cdot \sum_{i=1}^n \left(\sum_{r=1}^m \mathbf{1}_{r\in \ov{S}_i}\right)^2 & ~ \text{~by~Claim~\ref{cla:gradient_bound}} \\
  \leq & ~ \frac{\eta^2}{ m }\cdot (\frac{ \sqrt{n} }{ \sqrt{m} } \| \err(k) \|_2 )^2 \cdot \sum_{i=1}^n m^2\cdot \min\{R^2, \exp(-b^2)\} & \text{~by~Eq.~\eqref{eq:bound_ov_S}}\\
  = & ~ \eta^2 n^2\cdot \min\{R^2, \exp(-b^2)\}\cdot \| \err(k) \|_2^2,
\end{align*}

\end{proof}

\subsection{Bounding the blowup by the prediction movement term 
}\label{sec:proof_c4}

The proof of the following Claim is quite standard and simple in literature, see Claim 3.14 in \cite{sy19}. We omit the details here.
\begin{claim}[Bounding the blowup by the prediction movement term]\label{cla:C4}
\begin{align*}
B_4 \leq \eta^2 n^2 \cdot \| \err(k) \|_2^2.
\end{align*}
\end{claim}

\subsection{Putting it all together}\label{sec:proof_c8.5}
The goal of this section to combine all the convergence analysis together.
\begin{lemma}[Convergence]\label{clm:choose_eta_R}
Let $\eta={\lambda }/({4n^2})$, $R=\lambda / (12n)$, let $b \in[0,n]$, and
\begin{align*}
m \geq \Omega(\lambda^{-4}n^4 b^2  \log^2(n/\rho) ),
\end{align*}
we have
\begin{align*}
\Pr \Big[ \| \err(t) \|_2^2 \leq (1-\eta \lambda/2)^t \cdot \| \err(0) \|_2^2 \Big] \geq 1- 2\rho.
\end{align*}
\end{lemma}

\begin{proof}
We know  with probability $\geq$ $1-2n^2\cdot \exp(-\Omega(m)\cdot \min\{R,  \exp(-b^2/2)\})-\rho$,
\begin{align*}
\| 
\err(t+1) \|_2^2 \leq \| \err(t) \|_2^2\cdot( 1 - \eta \lambda + 4\eta n\cdot \min\{R, \exp(-b^2/2)\}  + \eta^2 n^2 )),
\end{align*}
and we want to show that
\begin{align}
    & 1 - \eta \lambda + 4\eta n\cdot \min\{R, \exp(-b^2/2)\}  + \eta^2 n^2 \leq 1-\eta \lambda /2, ~~~\text{and}\label{eq:final_goal1}\\
    & 2n^2\cdot \exp(-\Omega(m)\cdot \min\{R,  \exp(-b^2/2)\}) \leq \rho.\label{eq:final_goal2}
\end{align}

Claim~\ref{clm:4.1} requires the following relationship between $D$ and $R$, 
\begin{align*}
 D=  \frac{4\sqrt{n}\|\err(0)\|_2}{\sqrt{m}\lambda} < R
\end{align*}
By Claim \ref{cla:yu0}, we can upper bound the prediction error at the initialization,
\begin{align*}
    \|\err(0)\|_2^2=O(nb^2\log^2(n/\rho)),
\end{align*}
Combining the above two equations gives
\begin{align}\label{eq:R_constraint}
    R>\Omega(\lambda^{-1} nm^{-1/2}b\log (n/\rho)).
\end{align}

Claim~\ref{cla:C1} (where $0<c<e$ is a constant) requires an upper bound on $R$,\footnote{Due to the relationship between $b$ and $\lambda$, we are not allowed to choose $b$ in an arbitrary function of $\lambda$. Thus, we should only expect to use $R$ to fix the problem.}
\begin{align}\label{eq:R_or_b_constraint}
    R\leq \frac{\lambda}{12n}.
\end{align}

Combing the lower bound and upper bound of $R$, it implies the lower bound on $m$ in our Lemma statement.

And Lemma~\ref{lem:3.1} also requires that
\begin{align}\label{eq:m_constraint}
    m = \Omega(\lambda^{-1} n\log (n/\rho)).
\end{align}
which is dominated by the lower bound on $m$ in our lemma statement, thus we can ignore it.

Lemma~\ref{lem:3.1} and Claim~\ref{clm:S_condition} require that
\begin{align}\label{eq:R_cons2}
    R<1/b.
\end{align}
which is equivalent to
\begin{align*}
    b < 12n/\lambda
\end{align*}

However, by Theorem~\ref{thm:sep}, it will always hold for any $b>0$.

Note that Eq.~\eqref{eq:final_goal1} can be rewritten as 
\begin{align*}
    4\eta n\cdot \min\{R, \exp(-b^2/2)\}  + \eta^2 n^2 \leq \eta \lambda /2.
\end{align*}
where it follows from  taking $\eta := {\lambda}/{(4n^2)}$ and  $R = \lambda /(12n)$.

Therefore, we can take the choice of the parameters $m, b, R$  and Eqs.~\eqref{eq:final_goal1},~\eqref{eq:final_goal2} imply
\begin{align*}
\Pr[ 
\| \err(t+1) \|_2^2 \leq & ~ ( 1 - \eta \lambda / 2 ) \cdot \| \err(t) \|_2^2   ] \geq 1 - 2 \rho.
\end{align*}

\end{proof}


\section{Combine}\label{sec:combine}

\begin{corollary}[Sublinear cost per iteration]
Let $n$ denote the number of points. Let $d$ denote the dimension of points. Let $\rho \in (0,1/10)$ denote the failure probability. Let $\delta$ be the separability of data points. For any parameter $\alpha \in (0,1]$, we choose $b = \sqrt{0.5 (1-\alpha) \log m }$, if
\begin{align*}
    m = \Omega( ( \delta^{-4} n^{10} \log^{4} (n/\rho) )^{1/\alpha} ) 
\end{align*}
then the sparsity is
\begin{align*}
    O(m^{\frac{3+\alpha}{4}}).
\end{align*}
Furthermore,
\begin{itemize}
    \item If we preprocess the initial weights of the neural network, then we choose $\alpha = 1 - 1/\Theta(d)$ to get the desired running time.
    \item If we preprocess the training data points, then we choose $\alpha$ to be an arbitrarily small constant to get the desired running time.
\end{itemize}
\end{corollary}
\begin{proof}

From Theorem~\ref{thm:sep}, we know
\begin{align*}
    \lambda \geq \exp(-b^2/2) \cdot \frac{\delta}{100 n^2}
\end{align*}
which is equivalent to
\begin{align*}
    \lambda^{-1} \leq \exp(b^2/2) \cdot \frac{100n^2}{\delta}.
\end{align*}

For convergence, we need 
\begin{align*}
    m = \Omega(\lambda^{-4} n^4 b^2 \log^2(n/\rho) )
\end{align*}
Since we know the upper bound of $\lambda^{-1}$, thus we need to choose 
\begin{align*}
    m = \Omega( \exp(4 \cdot b^2/2) \cdot \delta^{-4} \cdot n^{10} b^2 \log^2(n/\rho) )
\end{align*}

From sparsity, we have
\begin{align*}
O(m \cdot \exp(-b^2/2))
\end{align*}
Let us choose $b = \sqrt{ 0.5(1-\alpha) \log m }$, for any $\alpha \in (0,1]$.

For the lower bound on $m$, we obtain
\begin{align*}
    m \geq  ( \delta^{-4} n^{10} \log^4 (n/\rho) )^{1/\alpha}
\end{align*}

For the sparsity, we obtain
\begin{align*}
    m \cdot m^{-(1-\alpha)/4} = m^{\frac{3+\alpha}{4}}
\end{align*}

\end{proof}

\begin{theorem}[Main result, formal of Theorem~\ref{thm:main_informal_1} and \ref{thm:main_informal_2}]\label{thm:main_1}
Given $n$ data points in $d$-dimensional space.
Running gradient descent algorithm  on a two-layer ReLU (over-parameterized) neural network with $m$ neurons in the hidden layers is able to minimize the training loss to zero, let $\Tinit$ denote the preprocessing time and ${\cal C}_{\mathsf{iter}}$ denote the cost per iteration of gradient descent algorithm.
\begin{itemize}
    \item If we preprocess the initial weights of the neural network (Algorithm~\ref{alg:ds_for_w_training}), then 
    \begin{align*}
        \Tinit = O_d( m \log m), {\cal C}_{\mathsf{iter}} = \wt{O} (m^{1-\Theta(1/d)} nd ).
    \end{align*}

    \item If we preprocess the training data points (Algorithm~\ref{alg:ds_for_x_training}), then
    \begin{align*}
         \Tinit = O(n^d), {\cal C}_{\mathsf{iter}} = \wt{O} (m^{3/4+o(1)} nd ) .
    \end{align*}
    \end{itemize}
\end{theorem}

\section{Bounds for the Spectral Gap with Data Separation}\label{sec:separation}

\begin{theorem}[Formal version of Proposition~\ref{prop:sep}]\label{thm:sep}Let $x_1,\dots,x_n$ be points in $\R^d$ with unit Euclidean norm and $w\sim{\cal N}(0,I_d)$. Form the matrix $X \in \R^{n\times d}=[x_1~\dots~x_n]^\top$. Suppose there exists $\delta \in (0,\sqrt{2})$ such that 
\begin{align*}
    \min_{i \neq j \in [n]} \{\|x_i-x_j\|_2,\|x_i+x_j\|_2\}\geq \delta.
\end{align*}
Let $b\geq 0$. Recall the continuous Hessian matrix $H^{\cts}$ is defined by
\begin{align*}
    H^{\cts}_{i,j}:=\E_{w \sim \N(0,I)} \left[ x_i^\top x_j {\bf 1}_{ w^\top x_i \geq b, w^\top x_j \geq b } \right]~~~\forall (i,j)\in [n]\times [n].
\end{align*}
Let $\lambda:=\lambda_{\min}(H^{\cts})$. Then, we have
\begin{align}\label{eq:require}
    \exp(-b^2/2) \geq~ \lambda \geq~\exp(-b^2/2) \cdot \frac{\delta}{100n^2}.
\end{align}
\end{theorem}

\begin{proof} 
\textbf{Part 1: Lower bound.}

Define the covariance of the vector $\mathbf{1}_{Xw> b}\in \R^{n}$ as 
\begin{align*}
    \E_{w\sim \mathcal{N}(0, I_d)}\left[(\mathbf{1}_{Xw> b})(\mathbf{1}_{Xw> b})^\top\right].
\end{align*}
Then, $H^{\cts}$ can be written as
\begin{align*}
    H^{\cts} = \E_{w\sim \mathcal{N}(0, I_d)}\left[(\mathbf{1}_{Xw> b})(\mathbf{1}_{Xw> b})^\top\right] \circ XX^\top,
\end{align*}
where $A \circ B$ denotes the Hadamard product between $A$ and $B$.

By Claim~\ref{clm:eigen_min}, and since $\|x_i\|_2=1$ for all $i\in [n]$, we only need to show:
\begin{align*}
    \E_{w\sim \mathcal{N}(0, I_d)}\left[(\mathbf{1}_{Xw> b})(\mathbf{1}_{Xw> b})^\top\right]\succeq \exp(-b^2/2) \cdot \frac{\delta}{100n^2}\cdot I_n.
\end{align*}
Fix a unit length vector $a\in\R^n$. Suppose there exist constants $c_1,c_2$ such that
\begin{align}\label{eq:needed}
\Pr\left[\left|a^\top \mathbf{1}_{Xw>b}\right|\geq c_1\|a\|_\infty\right]\geq \frac{c_2\delta }{n}.
\end{align}
This would imply that 
\begin{align*}
    \E\left[\left(a^\top \mathbf{1}_{Xw>b}\right)^2\right]\geq &~  \E\left[\left|a^\top\mathbf{1}_{Xw>b}\right|\right]^2\\
    \geq &~ c_1^2\|a\|_\infty^2 (\frac{c_2\delta}{n})^2\\
    \geq &~ c_1^2c_2\frac{\delta}{n^2},
\end{align*}
where the first step follows from Jensen's inequality, the second step follows from Markov's inequality, the last step follows from $\|a\|_2=1$.

Since this is true for all $a$, we find Eq.~\eqref{eq:require} with $c_1^2c_2=\frac{1}{100}$ by choosing $c_1=1/2,c_2=1/25$ as described later. 

Hence, our goal is proving Eq.~\eqref{eq:needed}. 
Without loss of generality, assume $|a_1|=\|a\|_\infty$ and construct an orthonormal basis $Q\in \R^{d\times d}$ in $\R^d$ where the first column is equal to $x_1 \in \R^d$ and $Q=[x_1~\bar{Q}] \in \R^{d \times d}$. Note that $g=Q^\top w\sim{\cal N}(0,I_d)$ and we have
\begin{align*}
    w=Qg=g_1 x_1+\bar{Q}\bar{g},
\end{align*}
where $g=\begin{bmatrix}g_1\\\bar{g}\end{bmatrix} \in \R^d$ and the first step follows from $QQ^\top = I_d$.

For $0\leq \gamma\leq 1/2$, Gaussian small ball guarantees
\begin{align*}
    \Pr[ |g_1|\leq \gamma]\geq \frac{7\gamma}{10}.
\end{align*}
Then, by Theorem 3.1 in \cite{ls01} (Claim~\ref{clm:gaussain_anti_shift}), we have
\begin{align*}
    \Pr[|g_1 - b |\leq \gamma] \geq \exp(-b^2/2)\cdot \frac{7\gamma}{10}.
\end{align*}
Next, we argue that $z_i:=\langle \bar{Q}\bar{g},x_i\rangle$ is small for all $i\neq 1$. For a fixed $i\geq 2$, observe that
\begin{align*}
   z_i\sim {\cal N}(0,1- \langle x_1 , x_i \rangle ^2). 
\end{align*}
Let $\tau_{i,1} := \langle x_i , x_1 \rangle$. 

Note that $\delta$-separation implies 
\begin{align*}
    1-| \langle x_1 , x_i \rangle |=\frac{1}{2} \min\{\|x_1-x_i\|_2^2,\|x_1+x_i\|_2^2\} \geq \frac{\delta^2}{2} 
\end{align*}
Hence $|\tau_{i,1}|\leq 1-\delta^2/2$. 

Then, from Gaussian anti-concentration bound (Lemma~\ref{lem:anti_gaussian}) and variance bound on $z_i$, we have
\begin{align*}
    \Pr[|z_i|\leq |\tau_{i,1} | \gamma ]
    \leq & ~ \sqrt{\frac{{2}}{{\pi}}}\frac{|\tau_{i,1}|\gamma}{\sqrt{1-\tau_{i,1}^2}} \\
    \leq & ~ \frac{2\gamma}{\delta\sqrt{\pi}} \\
    \leq & ~ \frac{2\gamma}{ \delta},
\end{align*}
which implies that
\begin{align*}
    \Pr[|z_i -(1-\tau_{i,1})b|\leq |\tau_{i,1}| \cdot \gamma] \leq \Pr[|z_i|\leq \gamma]\leq \frac{2\gamma}{\delta}.
\end{align*}

Hence, by union bound,
\begin{align*}
    \Pr[\forall i \in \{2,\cdots,n\} : |z_i -(1-\tau_{i,1})b|\leq |\tau_{i,1}| \cdot \gamma] \geq
    1-n\frac{2\gamma}{\delta}
\end{align*}




Define ${\cal E}$ to be the following event: 
\begin{align*}
    {\cal E}:=\Big\{|g_1-b|\leq \gamma ~\text{and}~ |z_i -(1-\tau_{i,1})b|\leq |\tau_{i,1}| \cdot \gamma, ~~~\forall i \in \{ 2, \cdots, n \} \Big\}.
\end{align*}

Since $g_1 \in \R$ is independent of $\bar{g}$, we have
\begin{align*}
    \Pr[{\cal E}]= &~  \Pr[|g_1-b|\leq \gamma] \cdot \Pr[\forall i \in \{2,\cdots, n\} : |z_i -(1-\tau_{i,1})b|\leq |\tau_{i,1}| \cdot \gamma]\\
    \geq &~ \exp(-b^2/2)\cdot \frac{7\gamma}{10} \cdot (1-2n\gamma / \delta)\\
    \geq &~ \exp(-b^2/2)\cdot \frac{7\delta}{80n}.
\end{align*}
where the last step follows from choosing $\gamma := \frac{\delta}{4n} \in [0, 1/2]$.

To proceed, define
\begin{align*}
    f(g):= &~ \langle a, \mathbf{1}_{X w > b} \rangle \\
    = & ~ a_1\cdot \mathbf{1}_{g_1>b}+\sum_{i= 2}^n(a_i\cdot\mathbf{1}_{x_i^\top x_1 \cdot g_1+ x_i^\top \bar{Q}\bar{g}>b})\\
    = &~ a_1\cdot \mathbf{1}_{g_1>b}+\sum_{i= 2}^n(a_i\cdot\mathbf{1}_{\tau_{i,1} \cdot g_1+ x_i^\top \bar{Q}\bar{g}>b}).
\end{align*}
where the third step follows from $\tau_{i,1} = x_i^\top x_1$.

On the event ${\cal E}$, by Claim~\ref{clm:g1_z_b}, we have that $\mathbf{1}_{\tau_{i,1}\cdot g_1+z_i>b}=\mathbf{1}_{z_i>(1-
\tau_{i,1})b}$.

Hence, conditioned on ${\cal E}$, 
\begin{align*}
f(g)=a_1\mathbf{1}_{g_1>b}+\text{rest}(\bar{g}),
\end{align*}
where
\begin{align*}
\text{rest}(\bar{g}):=\sum_{i= 2}^n a_i\cdot\mathbf{1}_{x_i^\top \bar{Q}\bar{g}>(1-
\tau_{i,1})b}.
\end{align*}

Furthermore, conditioned on ${\cal E}$, $g_1,\bar{g}$ are independent as $z_i$'s are function of $\bar{g}$ alone. Hence, ${\cal E}$ can be split into two equally likely events that are symmetric with respect to $g_1$ i.e.~$g_1\geq b$ and $g_1< b$.

Consequently, 
\begin{align}
\Pr\Big[|f(g)|\geq\max\{|a_1\mathbf{1}_{g_1>b}+\text{rest}(\bar{g})|,|a_1\mathbf{1}_{g_1<b}+\text{rest}(\bar{g})|\} ~\Big|~ {\cal E}\Big]\geq 1/2
\end{align}


Now, using $\max \{ |a|,|b| \} \geq |a-b|/2$, we find
\begin{align*}
& ~ \Pr[|f(g)|\geq 0.5 |a_1| \cdot |\mathbf{1}_{g_1>b}-\mathbf{1}
_{g_1<b}| ~|~ {\cal E}] \\
= & ~ \Pr[|f(g)|\geq 0.5 |a_1|  ~|~ {\cal E} ] \\
= & ~ \Pr[|f(g)|\geq 0.5 \|a\|_\infty ~|~ {\cal E}]\\
\geq & ~ 1/2,
\end{align*}
where $|a_1| = \| a \|_{\infty}$.

This yields 
\begin{align*}
    \Pr[|f(g)|\geq\|a\|_\infty/2]\geq \Pr[{\cal E}]/2\geq \exp(-b^2/2)\cdot \frac{7\delta}{160n}.
\end{align*} 
\textbf{Part 2: Upper bound.}
\begin{align*}
    \lambda = &~ \lambda_{\min}(H^{cts})\\
    = &~ \min_{x\in \R^d: \|x\|_2=1} x^\top H^{cts} x\\
    \leq &~ e_1^\top H^{\cts} e_1\\
    = &~ (H^{\cts})_{1,1}\\
    = &~ \E_{w} \left[ x_1^\top x_1 {\bf 1}_{ w^\top x_1 \geq b, } \right]\\
    = &~ \Pr_{w}[w^\top x_1 \geq b]\\
    \leq &~ \exp(-b^2/2),
\end{align*}
where $e_1:=\begin{bmatrix}1 & 0 & \cdots & 0 \end{bmatrix}^\top$, and the sixth step follows from $\|x_1\|_2=1$, the last step follows from the concentration of Gaussian distribution. In Line 5 and 6 of the above proof, $w$ is sampled from ${\cal N}(0,I_d)$.
\end{proof}

\begin{claim}\label{clm:g1_z_b}
Suppose $|g_1 - b| \leq \gamma$. 
\begin{itemize}
    \item If $\tau_{i,1}>0$, then $|z_i -(1-\tau_{i,1})b| > + \tau_{i,1}\gamma$ implies that $\mathbf{1}_{\tau_{i,1}\cdot g_1+z_i>b}=\mathbf{1}_{z_i>(1-\tau_{i,1})b}$.
    \item If $\tau_{i,1}<0$, then $|z_i -(1-\tau_{i,1})b|>-\tau_{i,1}\gamma$ implies that $\mathbf{1}_{\tau_{i,1}\cdot g_1+z_i>b}=\mathbf{1}_{z_i>(1-\tau_{i,1})b}$.
\end{itemize}
That is, if $|z_i -(1-\tau_{i,1})b|>|\tau_{i,1}|\gamma$, then we have $\mathbf{1}_{\tau_{i,1}\cdot g_1+z_i>b}=\mathbf{1}_{z_i>(1-\tau_{i,1})b}$.
\end{claim}

\begin{proof}

{\bf Case 1.}
We can assume  $\tau_{i,1}>0$. 
By assumption, we know that $g_1\in [b-\gamma, b+\gamma]$. 

Consider the forward direction first. 

If $\tau_{i,1}g_1 + z_i > b$, then 
\begin{align*}
z_i > b-\tau_{i,1}(b+\gamma)=(1-\tau_{i,1})b - \tau_{i,1}\gamma.
\end{align*}
According to the range of $z_i$, it implies $z_i > (1-\tau_{i,1})b$.

Then, consider the backward direction. 

If $z_i > (1-\tau_{i,1})b$, then by the range of $z_i$, we have $z_i > (1-\tau_{i,1})b + \tau_{i,1}\gamma$. 

Hence, 
\begin{align*}
\tau_{i,1}g_1 + z_i > \tau_{i,1}(b-\gamma) + (1-\tau_{i,1})b + \tau_{i,1}\gamma = b.
\end{align*}

{\bf Case 2.} 
The $\tau_{i,1}<0$ case can be proved in a similar way.
\end{proof}

\section{Quantum Algorithm for Training Neural Network}\label{sec:quantum_alg}
In this section, we provide a  quantum-classical hybrid approach to train neural networks with truly sub-quadratic time per iteration. The main observation is that the classical HSR data structure can be replaced with the Grover's search algorithm in quantum.

We first state our main result in below, showing the running time of our quantum training algorithm:
\begin{corollary}[Main theorem]\label{cor:main_quantum}
Given $n$ data points in $d$-dimensional space.
Running gradient descent algorithm  on a two-layer, $m$-with, over-parameterized, and ReLU neural network will minimize the training loss to zero, let ${\cal C}_{\mathsf{iter}}$ denote the cost per iteration of gradient descent algorithm. Then, we have 
    \begin{align*}
    {\cal C}_{\mathsf{iter}} = \wt{O} (m^{9/10} nd ).
    \end{align*} 
by applying Grover's search algorithm for the neurons (Algorithm~\ref{alg:qc_alg_1}) or the input data points (Algorithm~\ref{alg:qc_alg_2}).
\end{corollary}

\begin{remark}
We remark that previous works (\cite{klp_cnn19, ahkz20}) on training classical neural networks use the quantum linear algebra approach, which achieves quantum speedup in the linear algebra operations in the training process. For example, \cite{klp_cnn19} used the block encoding technique to speedup the matrix multiplication in training convolutional neural network (CNN). \cite{ahkz20} used the quantum inner-product estimation to reduce each neuron's computational cost. One drawback of this approach is that the quantum linear algebra computation incurs some non-negligible errors. Hence, extra efforts of error analysis are needed to guarantee that the intermediate errors will not affect the convergence of their algorithms. 

Compared with the previous works, the only quantum component of our algorithm is Grover's search. So, we do not need to worry about the quantum algorithm's error in the training process. And we are able to use our fast training framework to exploit a sparse structure, which makes the Grover's search algorithm run very fast, and further leads to a truly sub-quadratic training algorithm. 
\end{remark}
\begin{remark}
We also remark the difference between two algorithms in this quantum section 
the first algorithm runs Grover's search for each data point to find the activated neurons, while the second one runs Grover's search for each neuron to find the data points that make it activated. The advantage of Algorithm~\ref{alg:qc_alg_2} is it uses less quantum resources, since its search space is of size $O(n)$ and the first algorithm's search space is of size $O(m)$.
\end{remark}

\begin{algorithm}[H] 
\caption{Quantum-Classical Hybrid Training Neural Network, Version 1}
\label{alg:qc_alg_1} 
\begin{algorithmic}[1]
    \Procedure{OraclePrep}{$i\in [n], t\in [T]$}
        \State Prepare the quantum query oracle $\mathcal{O}_{i,t}$ such that \Comment{Each oracle call takes $O(d)$ time}
        \begin{align*}
            {\cal O}_{i,t}: \ket{r}\ket{0}\mapsto \begin{cases}
                \ket{r}\ket{1} & \text{if}~ w_r(t)^\top x_i > b,\\
                \ket{r}\ket{0} & \text{otherwise}.
            \end{cases}
        \end{align*}
    \EndProcedure
	\Procedure{QTrainingAlgorithmI}{$\{x_i\}_{i\in [n]},\{y_i\}_{i\in [n]},n,m,d$}\Comment{Corollary~\ref{cor:main_quantum}}
	\State Sample $w(0)$ and $a$ according to def.~\ref{def:init}
    \State $b\gets \sqrt{0.4\log m}$.
	\For{$t=0 \to T$}
	
	    \State \blue{/*Quantum part*/}
	    \For{$i=1\to n$} \label{ln:q1_quamtum_loop}
	        \State ${\cal O}_{i,t}\gets \textsc{OraclePrep}(i, t)$ 
	        \State Use Grover's search with oracle ${\cal O}_{i,t}$ to find the set $\mathcal{S}_{i,\mathrm{fire}}\subset [m]$ \State \Comment{It takes $\wt{O}(\sqrt{m\cdot k_{i,t}}\cdot d)$ time}
	    \EndFor
	    \State \blue{/*Classical part*/}
	    \For{$i=1 \to n$} \label{ln:q1_classical_1}
		    \State $u(t)_i \leftarrow \frac{1}{ \sqrt{m} } \sum_{r\in \mathcal{S}_{i,\mathrm{fire}}} a_r \sigma_b(w_r(t)^\top x_i)$ \Comment{It takes $O(d\cdot k_{i,t})$ time}
		\EndFor
		\For{$i = 1 \to n$} \label{ln:q1_classical_2}
		    \For{$r \in {\cal S}_{i,\mathrm{fire}}$}
			    \State $P_{i,r} \leftarrow \frac{1}{\sqrt{m}} a_r \sigma_b'( w_r(t)^\top x_i )$ 
			\EndFor
		\EndFor
		\State $M\gets X\diag(y-u(t))$ \Comment{$M\in \R^{d\times n}$, it takes $O(n\cdot d)$ time}
		\State $\Delta W \leftarrow MP$\Comment{$\Delta W\in \R^{d\times m}$, it takes $O(d\cdot \nnz(P))$ time}\label{ln:q1_compute_delta}
		\State $W(t+1)\gets W(t)-\eta \cdot \Delta W$.
	\EndFor
	\State \Return $W$
	\EndProcedure
\end{algorithmic}
\end{algorithm}

\begin{algorithm}[H]
\caption{Quantum-Classical Hybrid Training Neural Network, Version 2.}\label{alg:qc_alg_2}
\begin{algorithmic}[1]
    \algrenewcommand\algorithmicprocedure{\textbf{procedure}}
	\Procedure{QTrainingAlgorithmII}{$\{x_i\}_{i\in[n]},\{y_i\}_{i\in [n]}$,$n$,$m$,$d$} \Comment{Corollary~\ref{cor:main_quantum}}
	\State Sample $w(0)$ and $a$ according to def.~\ref{def:init}
    \State $b\gets \sqrt{0.4\log m}$.
   
    \State \blue{/*Initialize $\wt{S}_{r,\mathrm{fire}}$ and $S_{i,\mathrm{fire}}$ */}
     \Comment{It takes $\sum_{r=1}^{m}\wt{O}((n\wt{k}_{r,t})^{1/2}d)  \leq O(m^{9/10}nd)$ time in total.} 
     \State  $\wt{S}_{r,\mathrm{fire}} \leftarrow \emptyset$ for $r \in [m]$. 
    \State $S_{i,\mathrm{fire}} \leftarrow \emptyset$ for $i \in [n]$. 
    \For{$r = 1 \to m$} 
        
        \State $\wt{S}_{r,\mathrm{fire}} \leftarrow$ use Grover's serach to find all $i\in [n]$ s.t. $\sigma_b(w_r(1)^\top x_i)\ne 0$. \label{ln:q2_search_init}
        \For{$i \in \wt{S}_{r,\mathrm{fire}}$}
            \State $S_{i, \mathrm{fire}}.\textsc{Add}(r)$
        \EndFor
    \EndFor\label{ln:q2_init_e_q}
    
    \State \blue{/*Iterative step*/}
	\For{$t=0 \to T$}
	    \For{$i=1 \to n$}
		    \State $u(t)_i \leftarrow \frac{1}{ \sqrt{m} } \sum_{r\in \mathcal{S}_{i,\mathrm{fire}}} a_r \cdot \sigma_b(w_r(t)^\top x_i)$ \Comment{It takes $O(d\cdot k_{i,t})$ time}
		\EndFor
		\State $P \leftarrow 0^{n \times m}$ \Comment{$P \in \R^{n \times m}$}
		\For{$i = 1 \to n$} \label{ln:q2_classical_3}
		    \For{$r \in {\cal S}_{i,\mathrm{fire}}$}
			    \State $P_{i,r} \leftarrow \frac{1}{\sqrt{m}} a_r \cdot \sigma_b'( w_r(t)^\top x_i )$ 
			\EndFor
		\EndFor
		\State $M\gets X\diag(y-u(t))$ \Comment{$M\in \R^{d\times n}$, it takes $O(n\cdot d)$ time}
		\State $\Delta W \leftarrow MP$\label{ln:q2_compute_delta_2}\Comment{$\Delta W\in \R^{d\times m}$, it takes $O(m^{4/5}nd)$ time}
		\State $W(t+1)\gets W(t)-\eta \cdot \Delta W$.
		\State \blue{/*Update $\wt{S}_{r,\mathrm{fire}}$ and $S_{i,\mathrm{fire}}$ step*/}
		\Comment{It takes $\wt{O}(m^{9/10}nd)$ time in total}
		\State $S_{[n], \mathrm{fire}} \leftarrow \cup_{i \in [n]}\mathcal{S}_{i, \mathrm{fire}}$
		\For{$r \in S_{[n], \mathrm{fire}}$} 
		\For{$i \in \wt{S}_{r,\mathrm{fire}}$}  \Comment{It takes $O(\wt{k}_{r, t})$ time} 
            \State $S_{i, \mathrm{fire}}.\textsc{Del}(r)$ 
        \EndFor
        \State $\wt{S}_{r,\mathrm{fire}} \leftarrow$  use Grover's search to find all $i\in [n]$ s.t. $\sigma_b(w_r(t+1)^\top x_i)\ne 0$.  \label{ln:q2_recompute}
        \For{$i \in \wt{S}_{r,\mathrm{fire}}$}  \Comment{It takes $O(\wt{k}_{r, t+1})$ time}
            \State $S_{i, \mathrm{fire}}.\textsc{Add}(r)$
        \EndFor
    \EndFor
	\EndFor
	\State \Return $W$ \Comment{$W \in \R^{d \times m}$}
	\EndProcedure
\end{algorithmic}
\end{algorithm}

We first state a famous result about the quadratic quantum speedup for the unstructured search problem using Grover's search algorithm.

\begin{theorem}[Grover's search algorithm~\cite{grover,bhmt02}]\label{thm:grover_algo}
Given access to the evaluation oracle for an unknown function $f:[n]\rightarrow \{0,1\}$ such that $|f^{-1}(1)|=k$ for some unknown number $k\leq n$, we can find all $i$'s in $f^{-1}(1)$ in $\wt{O}(\sqrt{nk})$-time quantumly.
\end{theorem}


\begin{lemma}[Running time]\label{lem:running_time_q1}
For $t= 0,1,\dots, T$, the time complexity of the $t$-th iteration in Algorithm~\ref{alg:qc_alg_1} is
\begin{align*}
    \wt{O} \Big( n d \sqrt{m}\cdot \max_{i\in [n]} \sqrt{k_{i,t}}\Big),
\end{align*}
where $k_{i,t}=|{\cal S}_{i, \mathrm{fire}}(t)|$.
\end{lemma}
\begin{proof}
We first consider the quantum part of the algorithm, which is dominated by the for-loop at Line~\ref{ln:q1_quamtum_loop}. For each $i\in [n]$, we need to find the set ${\cal S}_{i, \mathrm{fire}}(t)$ by Grover's search, which takes $\wt{O}(\sqrt{mk_{i,t}}\cdot {\cal T}_{\mathrm{oracle}})$ time. In our case, each oracle call takes $O(d)$ time. Hence, the quantum part's running time is
\begin{align*}
    \wt{O}\Big(\sum_{i=1}^n \sqrt{mk_{i,t}}\cdot d\Big) = \wt{O}\Big(nd\sqrt{m}\cdot \max_{i\in [n]}\sqrt{k_{i,t}}\Big).
\end{align*}

Then, consider the classical part of the algorithm. Since we already get the sets ${\cal S}_{i,\mathrm{fire}}$, the for-loop at Line~\ref{ln:q1_classical_1} takes $O(k_t d)$ time, and the for-loop at Line~\ref{ln:q1_classical_2} takes $O(k_t)$ time, where $k_t=\sum_{i=1}^n k_{i,t}\leq n \cdot \max_{i\in [n]}k_{i,t}$. Then, at Line~\ref{ln:q1_compute_delta}, we compute the matrix product $X\diag(y-u(t))P$. It's easy to see that $M = X\diag(y-u(t))$ can be computed in time $O(nd)$. Since $P$ is a sparse matrix, $MP$ can be computed in $O(d\cdot \nnz(P))=O(dk_t)$ time. Namely, we maintain a data structure for all the non-zero entries of $P$. Then calculate each row of $MP$ in time $O(\nnz(P))$.  Hence, the total running time of the classical part is $O(nd\cdot \max_{i\in [n]}k_{i,t})$.

Since $k_{i,t}\leq m$ for all $i\in [n]$, the running time per iteration of Algorithm~\ref{alg:qc_alg_1} is $\wt{O}(nd\sqrt{m}\cdot \max_{i\in [n]}\sqrt{k_{i,t}})$, which completes the proof of the lemma.
\end{proof}

The following lemma proves the running time of Algorithm~\ref{alg:qc_alg_2}.

\begin{lemma}\label{lem:running_time_q2}
For $t= 0,1,\dots, T$, the time complexity of the $t$-th iteration in Algorithm~\ref{alg:qc_alg_2} is
\begin{align*}
    \wt{O} \Big( \sqrt{n} d \cdot \sum_{r=1}^m \sqrt{\wt{k}_{r,t}}\Big),
\end{align*}
where $\wt{k}_{r,t}=|\wt{{\cal S}}_{r, \mathrm{fire}}|$ at time $t$.
\end{lemma}

\begin{proof}
For the quantum part, the difference is at Line~\ref{ln:q2_search_init}, where we use Grover's search to find the data points such that the $r$-th neuron is activated. By Theorem~\ref{thm:grover_algo}, it takes $\wt{O}((n\wt{k}_{r,0})^{1/2})$-time quantumly. And at Line~\ref{ln:q2_recompute}, we re-compute $\wt{\cal S}_{r,\mathrm{fire}}$, which takes $\wt{O}((n\wt{k}_{r,t+1})^{1/2})$-time quantumly. Thus, the quantum running time of Algorithm~\ref{alg:qc_alg_2} is $\wt{O}(\sum_{r\in [m]}(n\wt{k}_{r,t})^{1/2})$ per iteration.

The classical part is quite similar to Algorithm~\ref{alg:data_full}, which takes $O(nd\cdot \max_{i\in[n]} k_{i,t})$-time per iteration.

Therefore, the cost per iteration is $\wt{O} (\sum_{r\in [m]}(n\wt{k}_{r,t})^{1/2} )$, and the lemma is then proved.
\end{proof}

Combining Lemma~\ref{lem:running_time_q1} and Lemma~\ref{lem:running_time_q2} proves the main result of this section:
\begin{proof}[Proof of Corollary~\ref{cor:main_quantum}]
In Section~\ref{sec:combine}, we prove that $k_{i,t}=m^{4/5}$ with high probability for all $i\in [n]$ if we take $b=\sqrt{0.4\log m}$. Hence, by Lemma~\ref{lem:running_time_q1}, each iteration in Algorithm~\ref{alg:qc_alg_1} takes
\begin{align*}
    \wt{O} \Big( n d \sqrt{m}\cdot \max_{i\in [n]} \sqrt{k_{i,t}}\Big) = \wt{O}\Big( n d m^{9/10}\Big)
\end{align*}
time in quantum. On the other hand, by Lemma~\ref{lem:running_time_q2}, each iteration in Algorithm~\ref{alg:qc_alg_2} takes quantum time
\begin{align*}
    \wt{O}\Big(\sqrt{n}d\sum_{r\in [m]}\sqrt{\wt{k}_{r,t}}\Big) \leq &~  \wt{O}\Big(\sqrt{n}d\sqrt{m}\sum_{r\in [m]}\wt{k}_{r,t}\Big)\tag{Cauchy-Schwartz inequality.}\\
    = &~ \wt{O}\Big(\sqrt{n}d\sqrt{m}\big(\sum_{i\in [n]}k_{i,t}\big)^{1/2}\Big)\\
    = &~  \wt{O}\Big( n d m^{9/10}\Big),
\end{align*}
where 
the second step is by $\sum_{r\in[m]}\wt{k}_{r,t} = \sum_{i\in [n]}k_{i,t}$, which completes the proof of the corollary.
\end{proof}





\end{document}